\documentclass[11pt]{article}

\usepackage{microtype}
\usepackage{graphicx}
\usepackage{subfig}
\usepackage{booktabs} 

\usepackage{hyperref}

\usepackage{fullpage}
\usepackage[utf8]{inputenc} 
\usepackage{amsfonts}       
\usepackage{amsmath}
\usepackage{amssymb}
\usepackage{bbm}
\usepackage{multirow}
\usepackage{mathtools}
\usepackage{relsize}
\usepackage{xcolor}
\usepackage{natbib}

\newcommand{\revision}[1]{#1}

\def\b0{{0}}

\def\poly{\operatorname{\mathop{poly}}}

\def\RR{\mathbb{R}}

\def\>{\rangle}
\def\rank{\operatorname{\mathop{rank}}}
\def\vec{\operatorname{\mathop{vec}}}
\def\diag{\operatorname{\mathop{diag}}}

\def\Set#1{\left\{ #1 \right\}}
\def\Bigbar#1{\mathrel{\left|\vphantom{#1}\right.}}
\def\Setbar#1#2{\Set{#1 \Bigbar{#1 #2} #2}}

\newcommand{\E}{\mathbb{E}}

\newcommand{\distas}[1]{\mathbin{\overset{#1}{\sim}}}

\newtheorem{theorem}{Theorem}[section]
\newtheorem{lemma}[theorem]{Lemma}
\newtheorem{definition}[theorem]{Definition}

\newtheorem{corollary}[theorem]{Corollary}

\newtheorem{assumptions}[theorem]{Assumption}

\newenvironment{proof}{\par\noindent{\bf Proof:\ }}{\hfill$\Box$\\[2mm]}

\newcommand{\Id}{\mathbb{I}}
\newcommand{\littleO}[1]{o\left(#1\right)}
\newcommand{\bigO}[1]{\mathcal{O}\left(#1\right)}
\newcommand{\bigOmg}[1]{\Omega\left(#1\right)}
\newcommand{\bigTheta}[1]{\Theta\left(#1\right)}
\newcommand{\inner}[1]{\left\langle#1\right\rangle}

\newcommand{\bigexp}[1]{\exp\left(#1\right)}
\newcommand{\norm}[1]{\left\|#1\right\|}

\newcommand{\abs}[1]{\left|#1\right|}

\newcommand{\svmin}[1]{\sigma_{\rm min}\left(#1\right)}
\newcommand{\evmax}[1]{\lambda_{\rm max}\left(#1\right)}
\newcommand{\evmin}[1]{\lambda_{\rm min}\left(#1\right)}

\def\HWineq{Theorem 6.2.1 of \cite{vershynin2018high}}
\def\Bernstein{Theorem 2.8.1 of \cite{vershynin2018high}}
\def\Hoeff{Theorem 2.2.6 of \cite{vershynin2018high}}
\def\GaussOpNorm{Theorem 2.12 of \cite{Davidson2001}}
\def\MatrixChernoff{Theorem 1.1 of \cite{Tropp2011}}

\def\Lip{\mathrm{Lip}}
\def\op{\mathrm{op}}
\def\bydef{\mathrel{\mathop:}=}

\def\PP{\mathbb{P}}

\def\tr{\mathop{\rm tr}\nolimits}

\def\interior{\mathop{\rm int}\nolimits}

\def\argmax{\mathop{\rm arg\,max}\limits}
\def\argmin{\mathop{\rm arg\,min}\limits}
\def\sign{\mathop{\rm sign}\limits}
\def\min{\mathop{\rm min}\nolimits}
\def\max{\mathop{\rm max}\nolimits}

\begin{document}

\title{Tight Bounds on the Smallest Eigenvalue \\of the Neural Tangent Kernel for Deep ReLU Networks}

\author{Quynh Nguyen\thanks{MPI-MIS, Germany. Email: \texttt{quynhnguyenngoc89@gmail.com}.}\;,
\;\;Marco Mondelli\thanks{Institute of Science and Technology Austria (ISTA). Email: \texttt{marco.mondelli@ist.ac.at}.}\;,\;\;Guido Montufar\thanks{MPI-MIS, Germany and UCLA. Email: \texttt{montufar@math.ucla.edu}.}}

\date{\today\thanks{This version corrects a mistake in the argument of Lemma \ref{lem:Fk_vs_Fk_tilde} of the paper (with the same title) appeared at ICML 2021. The mistake also affects Lemma \ref{lem:svmin_Fk_tilde_Khatri_Rao}. These two Lemmas have been edited and the corresponding proofs corrected. All the other results remain untouched.}}

\maketitle

\begin{abstract}
    A recent line of work has analyzed the theoretical properties of deep neural networks via the Neural Tangent Kernel (NTK). 
    In particular, the smallest eigenvalue of the NTK has been related to the memorization capacity, 
    the global convergence of gradient descent algorithms and the generalization of deep nets. 
    However, existing results either provide bounds in the two-layer setting or assume that the spectrum of the NTK matrices
    is bounded away from 0 for multi-layer networks. 
    In this paper, we provide tight bounds on the smallest eigenvalue of NTK matrices for deep ReLU nets, 
    both in the limiting case of infinite widths and for finite widths. 
    In the finite-width setting, the network architectures we consider are fairly general:
    we require the existence of a wide layer with roughly order of $N$ neurons, 
    $N$ being the number of data samples; and the scaling of the remaining layer widths is arbitrary (up to logarithmic factors). 
    To obtain our results, we analyze various quantities of independent interest: 
    we give lower bounds on the smallest singular value of hidden feature matrices, 
    and upper bounds on the Lipschitz constant of input-output feature maps.
\end{abstract}

\section{Introduction}
Consider an $L$-layer ReLU network with feature maps $f_l:\RR^d\to\RR^{n_l}$ defined for every $x\in\RR^d$ as
\begin{align}\label{eq:def_feature_map}
    f_l(x)=\begin{cases}
	    x & l=0,\\
	    \sigma(W_l^T f_{l-1}) & l\in[L-1],\\
	    W_L^T f_{L-1} & l=L, 
        \end{cases}
\end{align}
where $W_l\in\RR^{n_{l-1}\times n_l}$, $\sigma(x)=\max(0,x)$ and, given an integer $n$, we use the shorthand $[n]=\{1, \ldots, n \}$. 
We assume that the network has a single output, namely $n_L=1$ and $W_L\in\RR^{n_{L-1}}.$
For consistency, let $n_0=d.$
Let $g_l:\RR^d\to\RR^{n_l}$ be the pre-activation feature map so that $f_l(x)=\sigma(g_l(x)).$
Let $(x_1,\ldots,x_N)$ be $N$ samples in $\RR^d$, 
$\theta=[\vec(W_1),\ldots,\vec(W_L)]$,
and $F_L(\theta)=[f_L(x_1),\ldots,f_L(x_N)]^T.$
Let $J$ be the Jacobian of $F_L$ with respect to all the weights:
\begin{equation}\label{eq:Jac}
    J =\left[\frac{\partial F_L}{\partial\vec(W_1)},\ldots,\frac{\partial F_L}{\partial\vec(W_L)}\right] \in\RR^{N\times\sum_{l=1}^Ln_{l-1}n_l}.
\end{equation}
If not mentioned otherwise, we will assume throughout the paper that all the partial derivatives 
are computed by the standard back-propagation with the convention that $\sigma'(0)=0$.
The empirical Neural Tangent Kernel (NTK) Gram matrix, denoted by $\bar{K}^{(L)} \in\RR^{N\times N}$, is defined as:
\begin{equation}\label{eq:NTKgramdef}
    \bar{K}^{(L)}
    =J J^T
    =\sum_{l=1}^{L} \left[\frac{\partial F_L}{\partial\vec(W_l)}\right] \left[\frac{\partial F_L}{\partial\vec(W_l)}\right]^T.
\end{equation}
As shown in \citep{JacotEtc2018},
when $(W_l)_{ij}\distas{}\mathcal{N}(0,1)$ for all $l\in[L]$ and $\min\Set{n_1,\ldots,n_{L-1}}\to\infty$,
the normalized NTK matrix converges in probability to a non-random limit, called the limiting NTK matrix:
\begin{align}\label{eq:plim}
    \left(\prod_{l=1}^{L-1}\frac{2}{n_l}\right) \bar{K}^{(L)} \xrightarrow{\enskip p \enskip} K^{(L)} .
\end{align}
A quantitative bound for the convergence rate is provided in \citep{AroraEtal2019}.
Several theoretical aspects of training neural networks have been related to the spectrum of the NTK matrices.
For instance, considering the square loss $\Phi(\theta)=\frac{1}{2}\norm{F_L-Y}_2^2$, then a simple calculation shows that
\begin{align}\label{eq:GD}
    \norm{\nabla\Phi(\theta)}_2^2
    &\geq \evmin{\bar{K}^{(L)}} 2\Phi(\theta) .
\end{align}
The idea is that, if the spectrum of $\bar{K}^{(L)}$ is bounded away from zero at initialization,
then under suitable conditions, one can show that this property continues to hold during training.
In that case, $\evmin{\bar{K}^{(L)}}$ from \eqref{eq:GD} can be replaced by a positive constant, 
and thus minimizing the gradient on the LHS will drive the loss to zero.
This property, together with other smoothness conditions of the loss,
has been used for proving the global convergence of gradient descent in many prior works:
\citep{DuEtal2018_ICLR, OymakMahdi2019, SongYang2020,wu2019global} consider two layer nets, 
\citep{AllenZhuEtal2018,DuEtal2019,zou2020gradient,ZouGu2019} consider deep nets with polynomially wide layers,
and most recently \citep{QuynhMarco2020} consider deep nets with one wide layer of linear width followed by a pyramidal shape.
Beside optimization, the smallest eigenvalue of the NTK has been used to prove generalization bounds \citep{arora2019fine,Andrea2020}
and memorization capacity \citep{Andrea2020}.
All these analyses show that understanding the scaling of the smallest eigenvalue of the NTK is a problem of fundamental importance.

The recent work \citep{fan2020spectra} characterizes the full spectrum of
the limiting NTK via an iterated Marchenko-Pastur map.
Yet, this does not have implications on the scaling of any individual eigenvalue.
\citep{Andrea2020} gives a quantitative lower bound on $\evmin{\bar{K}^{(L)}}$ 
in a regime in which the number of parameters scales linearly with $N.$
This result is particularly interesting but currently restricted to a two-layer setup.
To our knowledge, for multi-layer architectures, 
the fact that the spectrum of the NTK is bounded away from zero is a typical working assumption \citep{DuEtal2019, HuangYau2020}.

\paragraph{Main contributions.} 
The aim of this paper is to provide tight lower bounds on the smallest eigenvalues of the empirical NTK matrices for deep ReLU networks.

First, we consider the asymptotic setting.
For i.i.d.\ data from a class of distributions that satisfy a Lipschitz concentration property
and for $(W_l)_{ij}\distas{}\mathcal{N}(0,1)$,
we show that the smallest eigenvalue of the limiting NTK matrix scales as
\begin{align}\label{eq:fast1}
    L\mathcal{O}(d) \geq 
    \evmin{K^{(L)}}
    \geq \Omega(d),
\end{align}
where $d$ captures the scaling of the average $L_2$ norm of the data 
\footnote{As introduced later, $d$ is also the input dimension. However, only the scaling of the data matters for our bounds.}.
This result is proved in our Theorem \ref{thm:limiting_NTK}.

Next, we consider networks with large but {\em finite} widths, and fixed depth.
Let $\xi_l$ be an auxiliary variable which takes value $1$ if $n_l=\tilde{\Omega}(N)$ and $0$ otherwise, 
where $N$ is the number of data points and $\tilde{\Omega}$ neglects logarithmic factors.
Then for $(W_l)_{ij}\distas{}\mathcal{N}(0,\beta_l^2)$, we show that
\begin{align}\label{eq:fast2}
    &\mathcal{O}\left(\left(d\prod_{l=1}^{L-1}n_l\right)\left(\prod_{l=1}^{L}\beta_l^2\right)\left(\sum_{l=1}^{L}\beta_l^{-2}\right)\right)
    \geq \evmin{\bar{K}^L} \nonumber \\
    &\geq \Omega\left(\left(d\prod_{l=1}^{L-1}n_l\right)\left(\prod_{l=1}^{L}\beta_l^2\right)\left(\sum_{l=1}^{L}\xi_{l-1}\beta_l^{-2}\right)\right) .
\end{align}
This is proved in Theorem \ref{thm:empirical_Jacobian}.
Our result directly implies that the spectrum of the NTK matrix is bounded away from zero whenever the network contains one wide layer of order $N$. 
This holds regardless of the position of the wide layer and the widths of the remaining ones (up to log factors).
The last property allows for networks with bottleneck layers.

Comparing the lower and upper bounds of \eqref{eq:fast2}, we note that they only differ in the scaling of 
$\sum_{l=1}^{L}\beta_l^{-2}$ and $\sum_{l=1}^{L}\xi_{l-1}\beta_l^{-2}.$
Let $k=\argmin\nolimits_{l\in[L-1]}\beta_l.$
Then, as long as $\xi_{k-1}=1$, both the sums will scale as $\beta_k^{-2}$.
In that case, the lower bound in \eqref{eq:fast2} is tight (up to a multiplicative constant).
For instance, this occurs if \emph{(i)} the network has one wide layer with $\tilde{\Omega}(N)$ neurons, 
and \emph{(ii)} it is initialized under He's initialization (i.e., $\beta_{l}=\sqrt{2/n_{l-1}}$)
or LeCun's initialization (i.e., $\beta_{l}=1/\sqrt{n_{l-1}}$) \citep{XavierBengio2010,he2015delving,lecun2012efficient}.
Note also that our bound for finite widths is consistent with the asymptotic one in \eqref{eq:fast1} (except that we do not track the dependence on $L$ in \eqref{eq:fast2}).

During the proof of our main theorems, we obtain other intermediate results which could be of independent interest:
\begin{itemize}
    \item We give a tight bound on the smallest singular value of feature matrices $F_k=[f_k(x_1), \ldots, f_k(x_N)]^T\in\RR^{N\times n_k},$ for $k\in[L-1]$. Our analysis requires only  a single wide layer, i.e.\ $n_k=\tilde{\Omega}(N)$, while all the previous layers can have {\em sublinear} widths.
    \item We obtain a new bound on the Lipschitz constant of the feature maps $f_k$'s for random Gaussian weights. 
    This bound is tighter than the one typically appearing in the literature (as given by the product of the operator norms of all the layers).
    The proof exploits a novel characterization of the Lipschitz constant of these maps, and leverages existing bounds on the number of activation patterns of deep ReLU nets.
\end{itemize}
This analysis allows us to prove the main results for a fairly general class of network shapes: 
there exists a layer with order  of $N$ neurons in an {\em arbitrary} position, and all the remaining layers can have {\em sublinear} widths, see Figure \ref{fig:net}.
No special ordering or relation between the scalings of these layers is needed. This goes beyond the setting of the typical NTK regime, where all the layers of the network have $\poly(N)$ neurons.

\begin{figure}[t]
    \centering
    \includegraphics[width=.5\columnwidth]{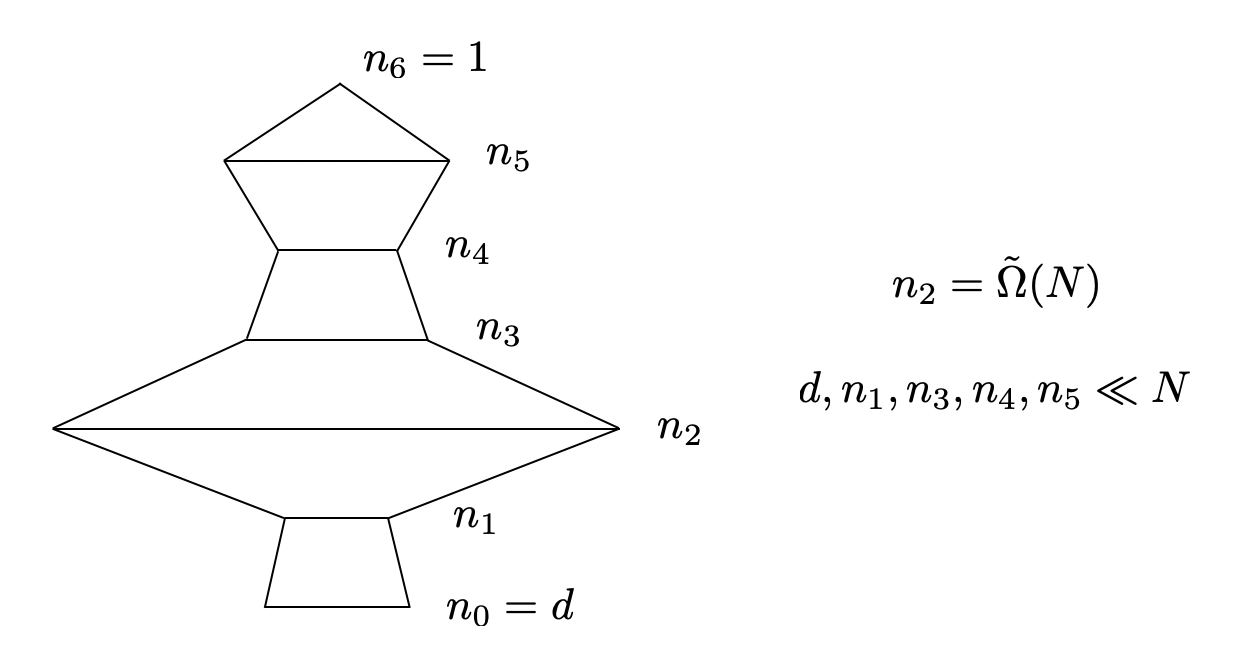}
\caption{Illustration of a network architecture to which our results can be applied (and that does not fall in the typical NTK regime).}\label{fig:net}
\end{figure}

\section{Preliminaries}\label{sec:setting}

\paragraph{Notations.} The following notations are used throughout the paper: given two integers $n<m$, let $[n, m]=\{n, n+1, \ldots, m\}$;
$X=[x_1,\ldots,x_N]^T\in\RR^{N\times d}$; the feature matrix at layer $l$ is
$F_l=[f_l(x_1),\ldots,f_l(x_N)]^T\in\RR^{N\times n_l}$; 
the centered feature matrices are $\tilde{F}_l=F_l-\E_X[F_l]$ for $l\in[L-1]$,
where the expectation is taken over all the samples; 
$\Sigma_l(x)=\diag([\sigma'(g_{l,j}(x))]_{j=1}^{n_l})$ for $l\in[L-1]$,
where $g_{l,j}(x)$ is the pre-activation neuron. Given two matrices $A,B\in\RR^{m\times n}$, we denote by $A\circ B$ their Hadamard product, 
and by $A\ast B=[(A_{1:}\otimes B_{1:}),\ldots,(A_{m:}\otimes B_{m:})]^T\in\RR^{m\times n^2}$ their row-wise Khatri-Rao product. 
Let $\norm{A}_{\op}$ be the operator norm of the matrix $A.$
Given a p.s.d.\ matrix $A$, we denote by $\sqrt{A}$ its square root (i.e. $\sqrt{A}=\sqrt{A}^T$ and $\sqrt{A}\sqrt{A}=A$).
We denote by $\norm{f}_{\Lip}$ the Lipschitz constant of the function $f$. 
All the complexity notations $\Omega(\cdot)$ and $\mathcal{O}(\cdot)$ are understood for sufficiently large $N,d,n_1,n_2,\ldots,n_{L-1}$. 
If not mentioned otherwise, the depth $L$ is considered a constant.

\noindent {\bf Hermite expansion.} 
Our bounds depend on the $r$-th Hermite coefficient of the ReLU activation function $\sigma$.
Let us denote it by $\mu_r(\sigma).$
By standard calculations, we have for any even integer $r\geq 2$, 
\begin{equation}\label{eq:HermiteReLU}
    \mu_r(\sigma) = \frac{1}{\sqrt{2\pi}} (-1)^{\frac{r-2}{2}}\frac{(r-3)!!}{\sqrt{r!}}.
\end{equation}

\noindent {\bf Weight and data distribution.} We consider the setting where both the weights of the network and the data are random.
In particular, $(W_l)_{i,j}\distas{}_{\rm i.i.d.}\mathcal{N}(0,\beta_l^2)$ for all $l\in[L], i\in[n_{l-1}], j\in[n_l],$
where the variable $\beta_l$ may depend on layer widths. 
Throughout the paper, we let $(x_1,\ldots,x_N)$ be $N$ i.i.d.\ samples from a data distribution, say $P_X$, such that the following conditions are satisfied.

\begin{assumptions}[Data scaling]\label{ass:data_dist}
 The data distribution $P_X$ satisfies the following properties:
    \begin{enumerate}
	\item $\int \norm{x}_2 dP_X(x)=\Theta(\sqrt{d}).$
	\item $\int \norm{x}_2^2 dP_X(x)=\Theta(d).$
	\item $\int \norm{x-\int x'\, dP_X(x')}_2^2 dP_X(x)=\bigOmg{d}.$
    \end{enumerate}
\end{assumptions}
These are just scaling conditions on the data vector $x$ or its centered counterpart $x-\E x$. 
We remark that the data can have any scaling, but in this paper we fix it to be of order $d$ for convenience.
We further assume the following condition on the data distribution.

\begin{assumptions}[Lipschitz concentration]\label{ass:data_dist2}
 The data distribution $P_X$ satisfies the \emph{Lipschitz concentration property}. Namely, for every Lipschitz continuous function $f:\RR^d\to\RR$, there exists an absolute constant $c>0$ such that, for all $t>0$,
$$
\PP\left(\abs{f(x)-\int f(x')\, dP_X(x')}>t\right)
\leq 2e^{-ct^2 / \norm{f}_{\Lip}^2}.
$$
\end{assumptions}
In general, Assumption \ref{ass:data_dist2} covers the whole family of distributions which satisfies the log-Sobolev inequality with a dimension-independent constant (or distributions with log-concave densities).
This includes, for instance, the standard Gaussian distribution, the uniform distribution on the sphere, 
or uniform distributions on the unit (binary or continuous) hypercube \citep{vershynin2018high}.
Let us remark that the coordinates of a random sample need not be independent under the above assumptions.
Note also that, by applying a Lipschitz map to the data, Assumption \ref{ass:data_dist2} still holds. 
Thus, data produced via a Generative Adversarial Network (GAN) fulfills our assumption, see \citep{seddik2020random}.


\section{Limiting NTK with All Wide Layers}\label{sec:NTK_infinite}

This section provides tight bounds on the smallest eigenvalue of the {\em limiting} NTK matrix $K^{(L)}\in\RR^{N\times N}$ from \eqref{eq:plim}.
As shown in \citep{JacotEtc2018}, one can compute this matrix recursively as follows, for all $l\in[2,L]$:
    \begin{equation}\label{eq:limNTK}
    \begin{split}
	&K^{(1)}_{ij}=G^{(1)},\\
	&K^{(l)}_{ij}=K^{(l-1)}_{ij}\, \dot{G}^{(l)}_{ij} + G^{(l)}_{ij},\\
	&\dot{G}^{(l)}_{ij} = 2\, \E_{(u,v)\distas{}\mathcal{N}(0, A^{(l)}_{ij})} [\sigma'(u)\sigma'(v)] , 
    \end{split}
    \end{equation}
where the matrices $G^{(l)}\in\RR^{N\times N}$ and $A^{(l)}_{ij}\in\RR^{2\times 2}$ are given by, for all $l\in[2,L]$,
    \begin{align}\label{eq:limNTK2}
    \begin{split}
	&G^{(1)}_{ij}=\inner{x_i, x_j},\\
	&A^{(l)}_{ij} = \begin{bmatrix}G^{(l-1)}_{ii} & G^{(l-1)}_{ij}\\G^{(l-1)}_{ji} & G^{(l-1)}_{jj}\end{bmatrix},\\
	&G^{(l)}_{ij} = 2\, \E_{(u,v)\distas{}\mathcal{N}(0, A^{(l)}_{ij})} [\sigma(u)\sigma(v)] , 
    \end{split}
    \end{align}

In order to prove our main result of this section,  we first need to rewrite the entry-wise formula of the NTK \eqref{eq:limNTK} in a more compact form. In particular, the following lemma provides a helpful characterization of the limiting NTK matrix.

\begin{lemma}\label{lem:limNTK_matform}
    The following holds for the matrices \eqref{eq:limNTK}-\eqref{eq:limNTK2}:
    \begin{align}
	&G^{(1)}=XX^T,\nonumber\\
	&G^{(2)} = 2\,\E_{w\distas{}\mathcal{N}(0,\, \Id_d)} \left[\sigma(Xw)\sigma(Xw)^T\right],\nonumber\\
	&G^{(l)} = 2\,\E_{w\distas{}\mathcal{N}(0,\, \Id_N)}\hspace{-.2em} \left[\sigma\left(\hspace{-.2em}\sqrt{G^{(l-1)}}\,w\right) \sigma\left(\hspace{-.2em}\sqrt{G^{(l-1)}}\,w\right)^T\right]\hspace{-.2em},\nonumber\\
	&\textrm{for } l\in[3,L].\label{eq:limGlnew}
    \end{align}
    \begin{align}
	&K^{(1)}=G^{(1)},\nonumber\\
	&K^{(l)} = K^{(l-1)} \circ \dot{G}^{(l)} + G^{(l)}, \quad\forall\,l\in[2,L],\label{eq:Kl_recurse} \\
	&\dot{G}^{(l)}=2\,\E_{w\distas{}\mathcal{N}(0,\, \Id_N)}\hspace{-.3em} \left[\sigma'\left(\hspace{-.3em}\sqrt{G^{(l-1)}}\,w\right)\hspace{-.2em} \sigma'\left(\hspace{-.3em}\sqrt{G^{(l-1)}}\,w\right)^T\right] \hspace{-.3em}\nonumber,\\
	&\textrm{for } l\in[2,L] \nonumber.
    \end{align}
    Moreover, we have
    \vspace{-5pt}
    \begin{align}\label{eq:KL_unroll}
	K^{(L)} = G^{(L)}\!+\!\sum_{l=1}^{L-1} G^{(l)} \circ \dot{G}^{(l+1)}\circ\ldots\circ\dot{G}^{(L)}.
    \end{align}
\end{lemma}
        \vspace{-5pt}
\begin{proof}
    Fix $l\in[2,L]$, and let $B=\sqrt{G^{(l-1)}}.$ 
    Then, 
    the equation \eqref{eq:limGlnew} can be rewritten as
    \begin{align*}
	G^{(l)}_{ij}
	=2\,\E_{w\distas{}\mathcal{N}(0,\, \Id_N)} \left[\sigma\left(\inner{B_{i:},w}\right) \sigma\left(\inner{B_{j:},w}\right)\right].
    \end{align*}
    Let $u=\inner{B_{i:},w}$ and $v=\inner{B_{j:},w}.$
    Then, one has $(u,v)\distas{}\mathcal{N}\left(0,\begin{bmatrix}G^{(l-1)}_{ii} & G^{(l-1)}_{ij}\\G^{(l-1)}_{ji} & G^{(l-1)}_{jj}\end{bmatrix}\right),$ which suffices to prove the expressions for $G^{(l)}$.
    A similar argument applies to $\dot{G}^{(l)}.$
    The equation \eqref{eq:KL_unroll} is obtained by unrolling \eqref{eq:Kl_recurse}.
\end{proof}
We are now ready to state the main result of this section. For space reason, a proof sketch is given below, and the full proof is deferred to Appendix \ref{app:prooflim}.
\begin{theorem}[Smallest eigenvalue of limiting NTK]\label{thm:limiting_NTK}
    Let $\Set{x_i}_{i=1}^{N}$ be a set of i.i.d.\ data points from $P_X$, where $P_X$ has zero mean and satisfies the Assumptions \ref{ass:data_dist} and \ref{ass:data_dist2}. 
    Let $K^{(L)}$ be the limiting NTK recursively defined in \eqref{eq:limNTK}. 
    Then, for any even integer constant $r\ge 2$, 
    we have w.p.\ at least $ 1 - Ne^{-\bigOmg{d}} - N^2e^{-\bigOmg{dN^{-2/(r-0.5)}}}$ that
        \vspace{-5pt}
    \begin{align}\label{eq:scalinginf}
	L\mathcal{O}(d)  
	\geq \evmin{K^{(L)}} \geq
	\mu_r(\sigma)^2\; \Omega(d), 
    \end{align}
    where $\mu_r(\sigma)$ is the $r$-th Hermite coefficient of the ReLU function given by \eqref{eq:HermiteReLU}.   
\end{theorem}
\begin{proof}
    Recall that for two p.s.d.\ matrices $P$ and $Q$, 
    it holds $\evmin{P\circ Q}\geq\evmin{P}\min_{i\in[n]}Q_{ii}$ \citep{schur1911bemerkungen}.
    By applying this inequality to the formula for the matrix $K_L$ in Lemma \ref{lem:limNTK_matform}, 
    and exploiting the fact that $\dot{G}^{(p)}_{ii}=1$ for all $p\in[2,L],i\in[N]$, 
    we obtain that $\evmin{K^{(L)}} \geq \sum_{l=1}^{L} \evmin{G^{(l)}}.$
    By using the Hermite expansion and homogeneity of ReLU, 
    one can bound $\evmin{G^{(l)}}$ in terms of $\evmin{\left((G^{(l-1)})^{*r}\right)\left((G^{(l-1)})^{*r}\right)^T}$, 
    for any integer $r>0$, where $(G^{(l-1)})^{*r}$ denotes the $r$-th Khatri Rao power of $G^{(l-1)}$. 
    Iterating this argument, it suffices to bound $\evmin{(X^{*r})(X^{*r})^T}.$ 
    This can be done via the Gershgorin circle theorem, and by using Assumptions \ref{ass:data_dist}-\ref{ass:data_dist2}.
\end{proof}
Let us make a few remarks about the result of Theorem \ref{thm:limiting_NTK}.
First, the probability can be made arbitrarily close to $1$ as long as $N$ does not grow super-polynomially in $d.$
Second, the $\Omega$ and $\mathcal{O}$ notations in \eqref{eq:scalinginf} do not hide any other dependencies on the depth $L.$
Finally, the proof of the theorem can be extended to other types of architectures, such as ResNet.

As mentioned in the introduction, non-trivial lower bounds on the smallest eigenvalue of the NTK have been used 
as a key assumption for proving optimization and generalization results in many previous works,
see e.g. \citep{arora2019fine,ChenEtal2020, DuEtal2018_ICLR} for shallow models and \citep{DuEtal2019, HuangYau2020} for deep models.
While quantitative lower bounds have been developed for shallow networks \citep{GhorbaniEtal2020},
this is the first time, to the best of our knowledge, that these bounds are proved for deep ReLU models.

For finite-width networks, when all the layer widths are sufficiently large, one would expect that, at initialization,
the smallest eigenvalue of the NTK matrix \eqref{eq:NTKgramdef} has a scaling similar to that given by Theorem \ref{thm:limiting_NTK}. 
A quantitative result can be obtained whenever the convergence rates of $\bar{K}^{(L)}$ to $K^{(L)}$  is available.
For instance, by using Theorem 3.1 of \citep{AroraEtal2019}, one has that, for $(W_l)_{ij}\distas{}\mathcal{N}(0,1),$
        \vspace{-3pt}
\begin{align}
    \abs{\left(\prod_{l=1}^{L-1}\frac{2}{n_l}\right) \bar{K}^{(L)}_{ij} - K^{(L)}_{ij}}
    \leq (L+1)\epsilon,
\end{align}
provided that $\min_{l\in[L-1]} n_l=\bigOmg{\epsilon^{-4}{\textrm{poly}(L)}}.$
By taking $\epsilon=(2(L+1)N)^{-1}\evmin{K^{(L)}},$ 
it follows that $\norm{\left(\prod_{l=1}^{L}\frac{2}{n_l}\right) \bar{K}^{(L)}-K^{(L)}}_{F}\leq\evmin{K^{(L)}}/2$, and thus
        \vspace{-3pt}
\begin{align}
    \evmin{\left(\prod_{l=1}^{L}\frac{2}{n_l}\right)\bar{K}^{(L)}}\in\left[\frac{1}{2}, \frac{3}{2}\right] \evmin{K^{(L)}} .
\end{align}
By applying Theorem \ref{thm:limiting_NTK}, one concludes that
        \vspace{-3pt}
\begin{align}
    \evmin{\bar{K}^{(L)}}=\bigTheta{d\prod_{l=1}^{L-1} n_l}
\end{align}
if $\min_{l\in[L-1]}n_l=\bigOmg{N^4}.$
This condition can be potentially improved if a better convergence rate of the NTK is available, e.g.\ plugging in the bounds of \citep{Sam2021} may give $\Omega(N^2)$.
Nevertheless, this still raises two questions: \emph{(i)} can one further relax the current conditions on layer widths? 
And \emph{(ii)} is it necessary to require all the layers to be wide 
to get a similar lower bound on the smallest eigenvalue? We address these questions in the next section.
\section{NTK Matrix with a Single Wide Layer}\label{sec:NTK_finite}
In this section, we provide bounds on the smallest eigenvalue of the empirical NTK matrix for networks of finite widths and fixed depth.
The networks we consider have a single wide layer (or more generally, any given subset of layers) with width linear in $N$ (up to logarithmic factors),
while all the remaining layers can have poly-logarithmic scalings.
Let us highlight that the position of the wide layer can be anywhere 
between the input and output layer of the network.
This setting is more challenging and closer to practice than the typical NTK one
where all the layers are often required to be very large in $N$.
Our main result of this section is stated below. Its proof is given in Section \ref{subsec:proofth}.
\begin{theorem}[Finite-width scaling of NTK eigenvalue] \label{thm:empirical_Jacobian}
    Consider an $L$-layer ReLU network \eqref{eq:def_feature_map}.
    Let $\Set{x_i}_{i=1}^{N}$ be a set of i.i.d.\ data points from $P_X$, 
    where $P_X$ satisfies the Assumptions \ref{ass:data_dist}-\ref{ass:data_dist2},
    and let $\bar{K}^{(L)}$ be the NTK Gram matrix, as defined in \eqref{eq:NTKgramdef}.
    Let the weights of the network be initialized as
    $[W_l]_{i,j}\distas{}\mathcal{N}(0,\beta_l^2)$, for all $l\in[L].$
    Fix any $\delta>0$ and any even integer $r\ge 2$. For $k\in[L-1]$, let $\xi_k$ be $1$ if the following condition holds:
        \vspace{-5pt}
    \begin{align}
	&n_k = \bigOmg{N\log(N) \log\Big(\frac{N}{\delta}\Big)},\\
	&\prod_{l=1}^{k-2} \log(n_l)=\littleO{\min_{l\in[0,k-1]} n_l},
    \end{align}
    and let $\xi_k$ be $0$ otherwise.
    Let $\mu_r(\sigma)$ be given by \eqref{eq:HermiteReLU}.
    Then,
          \vspace{-5pt}
  \begin{align}\label{eq:Jacobian_lowerbound}
	\evmin{\bar{K}^{(L)}} 
	&\geq\sum_{k=2}^{L} \xi_{k-1}\; \mu_r(\sigma)^2\; 
	\bigOmg{d\prod_{l=1}^{L-1}n_l \prod_{\substack{l=1\\l\neq k}}^{L}\beta_l^2 } \nonumber\\
	&+ \evmin{XX^T} \bigOmg{\prod_{l=1}^{L-1}n_l \prod_{l=2}^{L}\beta_l^2 } 
    \end{align}
    w.p.\ at least
            \vspace{-5pt}
    \begin{align}\label{eq:prob}
	&1 \hspace{-.2em}-\hspace{-.2em} \delta \hspace{-.2em}-\hspace{-.2em} \sum_{k=1}^{L-1} \xi_k N^2 \bigexp{ \hspace{-.2em}- \bigOmg{ \frac{\min_{l\in[0,k-1]}n_l}{N^{2/(\revision{r/2}-0.1)} \prod_{l=1}^{k-2}\log(n_l)} } } \nonumber\\
	&- N \sum_{l=1}^{L-1} \bigexp{-\bigOmg{n_l}} - N \exp(-\bigOmg{d}) .
    \end{align}
    Moreover, we have that, w.p.\ at least $1 - \sum_{l=1}^{L-1} \bigexp{-\bigOmg{n_l}} - \exp(-\bigOmg{d})$,
             \vspace{-5pt}
   \begin{align}\label{eq:Jacobian_upperbound}
	\evmin{\bar{K}^{(L)}} 
	\leq\sum_{k=1}^{L}\,
	\bigO{d\prod_{l=1}^{L-1}n_l \prod_{\substack{l=1\\l\neq k}}^{L}\beta_l^2 }. 
    \end{align}
\end{theorem}

\begin{figure*}[t]
    \centering
    \subfloat{\includegraphics[width=.45\columnwidth]{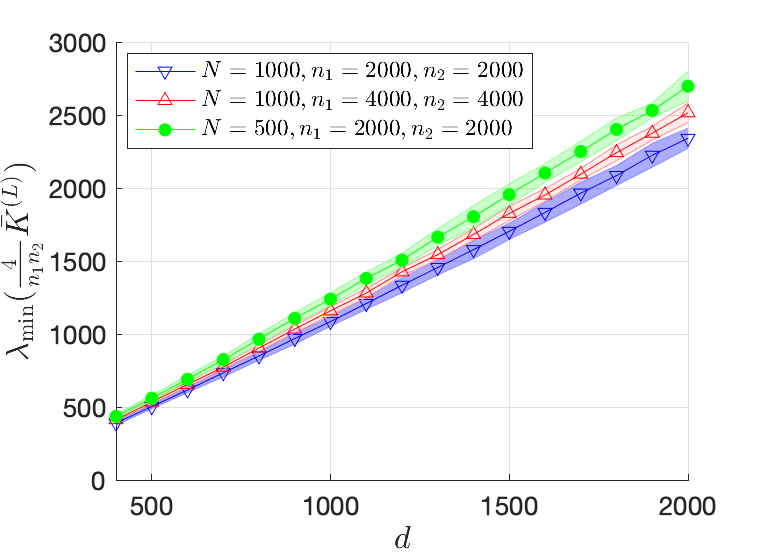}}\hspace{3em}
    \subfloat{\includegraphics[width=.45\columnwidth]{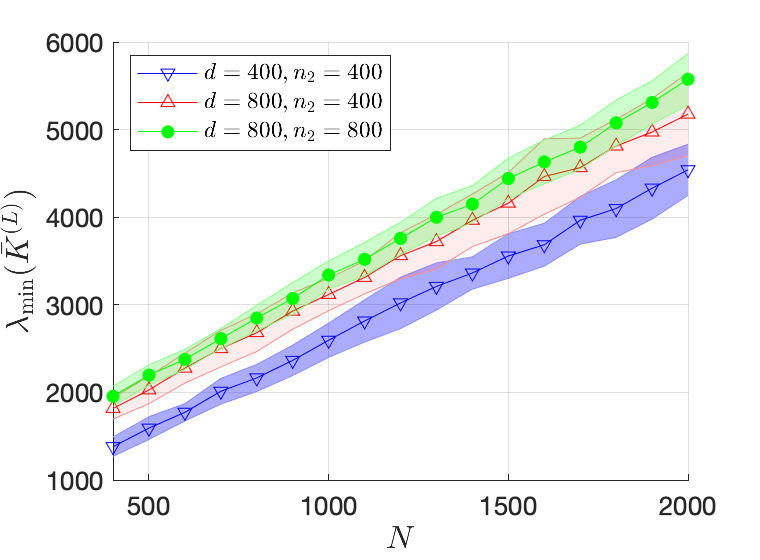}}
\caption{Scaling of the smallest eigenvalue of NTK matrices as a function of the input dimension $d$ (on the left) and of the number of samples $N$ (on the right). The theoretical results of Theorem \ref{thm:limiting_NTK} and \ref{thm:empirical_Jacobian} are in excellent agreement with the plot.}\label{fig:jac}

\end{figure*}

The two plots in Figure \ref{fig:jac} provide empirical evidence supporting our main results for $L=3$. We perform 50 Montecarlo trials, and report average and confidence interval at 1 standard deviation. On the left, we take $(W_l)_{i, j}\sim \mathcal N(0, 1)$, fix the parameters $(N, n_1, n_2)$, scale the NTK matrix by $\frac{4}{n_1 n_2}$ (see \eqref{eq:plim}), and plot $\lambda_{\rm min}(\frac{4}{n_1 n_2}\bar{K}^{(L)})$ as a function of $d$.  The three curves correspond to three different choices of $(N, n_1, n_2)$. As predicted by our Theorem \ref{thm:limiting_NTK}, the smallest eigenvalue of the NTK exhibits a linear dependence on $d$. On the right, we take $(W_l)_{i, j}\sim \mathcal N(0, 2/n_{l-1})$ (the popular He's initialization), fix $(d, n_2)$, set $n_1=8N$, and plot $\lambda_{\rm min}(\bar{K}^{(L)})$ as a function of $N$. The three curves correspond to three different choices of $(d, n_2)$. In this setting, there is a single wide layer and our Theorem \ref{thm:empirical_Jacobian} predicts that the smallest eigenvalue of the NTK scales linearly in the width of the wide layer (and hence linearly in $N$). This is in excellent agreement with the plot.


The results of both Theorem \ref{thm:limiting_NTK} and \ref{thm:empirical_Jacobian} rely on considering a single term in the sum over layers and a fixed $r$. However, we expect the gap due to this fact to be rather small: \emph{(i)} the Hermite coefficients of the ReLU decay quite slowly (see \eqref{eq:HermiteReLU}), 
so the dependence of the bounds in $r$ is mild; \emph{(ii)} we are mainly interested in networks with a single wide layer, 
and in this setting the sum is well approximated by the leading term. 
Taking into account more terms of the sum or more $r$ is an interesting problem for future work.
Unlike Theorem \ref{thm:limiting_NTK}, we do not track the dependence on $L$ in Theorem \ref{thm:empirical_Jacobian}, and therefore the constants implicit in $\Omega$ and $\mathcal O$ may depend on $L$.
One can see that the lower bound \eqref{eq:Jacobian_lowerbound} and the upper bound \eqref{eq:Jacobian_upperbound} will have the same scaling, that is
            \vspace{-5pt}
\begin{align}
    \left(d\prod_{l=1}^{L-1}n_l\right)\left(\prod_{l=1}^{L}\beta_l^2\right) \left(\min_{l\in[L]}\beta_l\right)^{-2},
\end{align}
provided that there exists a layer $k\in[L-1]$ such that $\xi_k=1$
and $\beta_{k+1}=\min_{l\in[L]}\beta_{l}.$
For instance, this holds if \emph{(i)} the network contains one wide hidden layer with $\tilde{\Omega}(N)$ neurons, 
and \emph{(ii)} it is initialized using the popular He's or LeCun's initialization 
(i.e., $\beta_{l}=c/\sqrt{n_{l-1}}$ for some constant $c$) \citep{XavierBengio2010,he2015delving,lecun2012efficient}.
In that case, the scaling of the lower bound \eqref{eq:Jacobian_lowerbound} is tight (up to a multiplicative constant).
Note also that the probability in \eqref{eq:prob} can be made arbitrarily close to $1$
provided that all the layers before the wide layer $k$ do not exhibit exponential bottlenecks in their widths.


In a nutshell, Theorem \ref{thm:empirical_Jacobian} shows (in a quantitative way) that the spectrum of the NTK matrix is bounded away from zero. The requirements on the network architecture are mild:  \emph{(i)} existence of a wide layer with $\tilde{\Omega}(N)$ neurons, and \emph{(ii)} absence of exponential bottlenecks before the wide layer.
This last condition means that after the wide layer(s), the widths of the network need not have any relation with each other, thus can scale differently.
This is a more general setting than the one considered in \citep{QuynhICML2019,QuynhICML2017,QuynhMarco2020} 
where the network has a single wide layer, which is then followed by a pyramidal shape (i.e. the widths are non-increasing towards the output layer).
Here, the pyramidal constraint is not needed.

Let us make a few remarks about the case of shallow nets ($L=2$) as
tight lower bounds on $\evmin{\bar{K}^{(L)}}$ have been also obtained in several recent works,
albeit for a different setting than the one in Theorem \ref{thm:empirical_Jacobian}.
In particular, \citep{Andrea2020} consider the regime where $n_0=\Omega(n_1)$ and $n_0n_1=\Omega(N)$,
whereas we consider $n_1=\Omega(N)$ and have little restrictions on $n_0$. 
\citep{OymakMahdi2019} give bounds for a similar regime to ours, 
but a possible generalization of their proof to the case of multi-layer networks 
would require all the layers to be wide with at least $\tilde{\Omega}(N)$ neurons. 
In contrast, Theorem \ref{thm:empirical_Jacobian} essentially requires an {\em arbitrary} single wide layer of width $\Omega(N)$, while all the remaining layers
can have almost any widths (up to log factors).
To obtain this, the proof of Theorem \ref{thm:empirical_Jacobian} requires lower bounds on the smallest eigenvalue of the intermediate feature matrices $F_k$'s
for networks with a single wide layer, and the Lipschitz constant of the intermediate feature maps,
which are not studied in the previous works.

Our Theorem \ref{thm:empirical_Jacobian} immediately implies that such a class of networks
can fit $N$ distinct data points arbitrarily well, for any {\em real} labels. 
The fact that the positive definiteness of the NTK implies a property on {\em memorization capacity} of neural nets 
has been already observed in \citep{Andrea2020}, albeit for a two-layer model. 
The following corollary provides a formal connection between the two for the case of deep nets, 
and it should be seen as a proof of concept. Its proof is given in Appendix \ref{app:zero_loss}. 
\begin{corollary}[Memorization capacity]\label{cor:zero_loss}
    Consider an $L$-layer ReLU network \eqref{eq:def_feature_map}.
    Let $\Set{x_i}_{i=1}^{N}$ be a set of i.i.d.\ data points from $P_X$, 
    where $P_X$ satisfies the Assumptions \ref{ass:data_dist}-\ref{ass:data_dist2}.
    Fix any $\delta, \delta'>0.$
    Assume that there exists a layer $k\in[L-1]$ such that
	$n_k = \bigOmg{N\log(N) \log\Big(\frac{N}{\delta}\Big)}$ and
	$\prod_{l=1}^{k-2} \log(n_l)=\littleO{\min_{l\in[0,k-1]} n_l}.$
    Then, it holds
    \begin{align*}
	\forall\, Y,\; \forall\epsilon>0,\; \exists\,\theta: \quad\norm{F_L(\theta)-Y}_2\leq\epsilon
    \end{align*}
    w.p.\ at least
	$ 1 - \delta - N^2 e^{ - \bigOmg{ \frac{\min_{l\in[0,k-1]}n_l}{N^{\delta'} \prod_{l=1}^{k-2}\log(n_l)} } } - N \sum_{l=1}^{L-1} e^{-\bigOmg{n_l}} - N e^{-\bigOmg{d}} $
    over the data.
\end{corollary}

In words, Corollary \ref{cor:zero_loss} shows that if a deep ReLU network contains a wide layer of order $\tilde{\Omega}(N)$ neurons,
then regardless of the position of this wide layer, and regardless of the widths of the remaining layers (up to log factors),
the network can approximate $N$ data points (with real labels) within arbitrary precision.
Here, the network has $\tilde{\Omega}(N)$ total parameters, which is known to be (nearly) tight for memorization capacity.
However, we remark that this is not optimal in terms of layer widths.
In particular, several recent works \citep{bartlett2019nearly,ge2019mildly,vershynin2020memory,yun2019small} 
show that under some other mild conditions (without the existence of a wide layer as in Corollary \ref{cor:zero_loss}),
$\Omega(N)$ parameters suffice for the network to memorize $N$ data points.
Nevertheless, let us remark some differences in terms of the setting between these results and the one in Corollary \ref{cor:zero_loss}:
\emph{(i)} prior works consider networks with biases while Corollary \ref{cor:zero_loss} consider nets with no biases, and 
\emph{(ii)} prior works consider data with {\em bounded} labels while Corollary \ref{cor:zero_loss} applies to {\em arbitrary} real labels.
For shallow networks (i.e.\ $L=2$), stronger memorization results than Corollary \ref{cor:zero_loss} have been achieved.
For instance, \citep{bubeck2020network} show that width $\Omega(N/n_0)$ suffices for a two-layer ReLU net to memorize $N$ arbitrary data points.
\citep{Andrea2020} show a similar result under an additional assumption (i.e.\ $n_0=\Omega(n_1)$ and $n_0n_1=\Omega(N)$), albeit for more general class of activations.

\subsection{Proof of Theorem \ref{thm:empirical_Jacobian}.}\label{subsec:proofth}
    By chain rules and some standard manipulations, we have
    \begin{align*}
	JJ^T = \sum_{k=0}^{L-1} F_{k} F_{k}^T \circ B_{k+1}B_{k+1}^T 
    \end{align*}
    where $B_k\in\RR^{N\times n_k}$ is a matrix whose $i$-th row is given by
    \begin{align*}
	(B_k)_{i:} = 
	\begin{cases}
	    \Sigma_{k}(x_i)\hspace{-.2em} \left(\prod_{l=k+1}^{L-1} W_l\Sigma_l(x_i)\right) \hspace{-.2em}W_L, &  \hspace{-.6em} k\in[L-2],\\
	    \Sigma_{L-1}(x_i)W_L, & \hspace{-.6em} k=L-1,\\
	    \frac{1}{\sqrt{N}}1_{N}, & \hspace{-.6em} k=L.
	\end{cases}
    \end{align*}

    For PSD matrices $P,Q\in\RR^{n\times n}$, it holds $\evmin{P\circ Q}\geq\evmin{P}\min_{i\in[n]}Q_{ii}$ \citep{schur1911bemerkungen}.
    Thus, 
    \begin{align}\label{eq:Jlb}
	\evmin{JJ^T} \geq
	\sum_{k=0}^{L-1} \evmin{F_k F_k^T}\; 
	\min\limits_{i\in[N]} \norm{(B_{k+1})_{i:}}_2^2.
    \end{align}
    We now bound every term on the RHS of \eqref{eq:Jlb}.
    Doing so requires a careful analysis of various quantities involving the hidden layers.
    This includes the smallest singular value of the feature matrices $F_k\in\RR^{N\times n_k}$,
    and the Lipschitz constant of the feature maps $f_k,g_k:\RR^d\to\RR^{n_k}.$ 
    As these results could be of independent interest, we put them separately in the following sections.
    In particular, our Theorem \ref{thm:bound_svmin_Fk} from the next section proves bounds for $\evmin{F_kF_k^T}.$
    To bound the norm of the rows of $B_{k+1}$, one can use the following lemma (for the proof, see Appendix \ref{app:norm_W_SIGMA_wL}).
    \begin{lemma}\label{lem:norm_W_SIGMA_wL}
	Fix any layer $k\in[L-2]$, and $x\distas{}P_X.$
	Then,
	\begin{align*}
	    \norm{\Sigma_{k+1}(x) \left(\prod_{l=k+2}^{L-1} W_l\Sigma_l(x)\right) W_L }_2^2\\
	    = \bigTheta{\beta_L^2\, n_{k+1} \prod_{l=k+2}^{L-1} n_l\beta_l^2},
	\end{align*}
	w.p.\ at least $1 - \sum_{l=1}^{L-1} \bigexp{-\bigOmg{n_l}} - \exp(-\bigOmg{d})$. 
	Here, we assume by convention that the product term $\prod_{l=k+2}^{L-1}(\cdot)$ is inactive for $k=L-2.$
    \end{lemma}
   By plugging the bounds of Lemma \ref{lem:norm_W_SIGMA_wL} and  Theorem \ref{thm:bound_svmin_Fk} into \eqref{eq:Jlb}, the lower bound in \eqref{eq:Jacobian_lowerbound} immediately follows.
    For the upper bound, note that
    \begin{align}\label{eq:Jub}
	\evmin{JJ^T} 
	\leq (JJ^T)_{11}
	= \sum_{k=0}^{L-1} \norm{(F_k)_{1:}}_2^2 \norm{(B_{k+1})_{1:}}_2^2 .
    \end{align}
    The second term in the RHS of \eqref{eq:Jub} can be bounded by using Lemma \ref{lem:norm_W_SIGMA_wL} above. To bound the first term, we note that $(F_k)_{1:}=f_k(x_1)$ and that, for every $0\leq k\leq L-1$, 
    \begin{equation}\label{eq:Jub2}
    \norm{f_k(x_1)}_2^2 = \bigTheta{ d\prod_{l=1}^k n_l \beta_l^2 },
    \end{equation}
    w.p.\ at least $1-\sum_{l=1}^{k} \bigexp{-\bigOmg{n_l}} - \exp(-\bigOmg{d})$. This last statement follows from Lemma \ref{lem:norm_feature_map} in Appendix \ref{app:more_results}. By plugging \eqref{eq:Jub2} and the bound of Lemma \ref{lem:norm_W_SIGMA_wL} into \eqref{eq:Jub}, the upper bound in \eqref{eq:Jacobian_upperbound} immediately follows.

\section{Smallest Singular Values of Feature Matrices}\label{sec:svmin_Fk}
As before, we assume throughout this section that $(W_l)_{ij}\distas{}\mathcal{N}(0,\beta_l^2)$ for $l\in[L],$ 
and the data points are i.i.d.\ from a distribution $P_X$ satisfying Assumption \ref{ass:data_dist} and \ref{ass:data_dist2}. 
Let us recall the definition of the feature matrix at some hidden layer $k$: $F_k=[f_k(x_1),\ldots,f_k(x_N)]^T\in\RR^{N\times n_k}.$
Our main result of this section is the following tight bound on the smallest singular values of these matrices.

\begin{theorem}[Smallest singular value of feature matrix]\label{thm:bound_svmin_Fk}
    Fix any $k\in[L-1]$ and any even integer constant $r\ge 2.$
    Let $\delta>0$ be given.
    Assume that
    \begin{align}
	&n_k = \bigOmg{N\log(N) \log\Big(\frac{N}{\delta}\Big)},\\
	&\prod_{l=1}^{k-2} \log(n_l)=\littleO{\min_{l\in[0,k-1]} n_l}.
    \end{align}
    Let $\mu_r(\sigma)$ be given by \eqref{eq:HermiteReLU}.
    Then, the smallest singular value of the feature matrix $F_k$ satisfies
    \begin{align*}
	\bigO{ \hspace{-.2em}d\prod_{l=1}^{k} n_l\beta_l^2 \hspace{-.2em}} 
	\ge \svmin{F_k}^2 
	\ge \mu_r(\sigma)^2\; \bigOmg{\hspace{-.2em} d\prod_{l=1}^{k} n_l\beta_l^2 \hspace{-.2em}}
    \end{align*}
    w.p.\ at least
    \begin{align*}
	&1 - \delta - N^2 \bigexp{ - \bigOmg{ \frac{\min_{l\in[0,k-1]}n_l}{N^{2/(\revision{r/2}-0.1)} \prod_{l=1}^{k-2}\log(n_l)} } }\\
	&- N \sum_{l=1}^{k-1} \bigexp{-\bigOmg{n_l}} - N \exp(-\bigOmg{d}) .
    \end{align*}
\end{theorem}

\noindent {\bf Proof of Theorem \ref{thm:bound_svmin_Fk}.}
First of all, the conditions of Theorem \ref{thm:bound_svmin_Fk} imply that $n_k\geq N$,
which further implies $\svmin{F_k}^2=\evmin{F_kF_k^T}.$
To bound this quantity, we first relate it to the smallest eigenvalue of the expected Gram matrix, 
namely $\E[F_kF_k^T]$, where the expectation is taken over $W_k.$
Note that $\E[F_kF_k^T]=n_k\E[\sigma(F_{k-1}w)]\sigma(F_{k-1}w)^T]$, where $w$ has the same distribution as any column of $W_k.$ 
This is formalized in the following lemma, which is proved in Appendix \ref{app:lem:svmin_Fk_chernoff}. 
\begin{lemma}\label{lem:svmin_Fk_chernoff}
    Let us define
    \begin{align}
	\lambda=\evmin{ \E_{w\distas{}\mathcal{N}(0,\beta_k^2\,\Id_{n_{k-1}})} [\sigma(F_{k-1}w) \sigma(F_{k-1} w)^T] }.
    \end{align}
    Fix any $\delta>0.$
    Assume that
    \begin{align*}
	n_k\geq\max\left(N,\;\; c\, Q \max\Big(1, \log(4Q)\Big) \log\frac{N}{\delta} \right) ,
    \end{align*}
    where $c$ is an absolute constant, and $Q\bydef \frac{\beta_k^2\norm{F_{k-1}}_F^2}{\lambda}.$
    Then, we have w.p.\ at least $1-\delta$ over $W_k$ that
    \begin{align*}
	\svmin{F_k}^2\geq \frac{n_k\lambda}{4} .
    \end{align*}
\end{lemma}
From here, it suffices to upper bound $\norm{F_{k-1}}_F^2$ and lower bound $\lambda.$
The first quantity can be bounded by using a standard induction argument over $k$. In particular, from Lemma \ref{lem:norm_feature_map} in Appendix \ref{app:more_results}, it follows that $\norm{F_{k-1}}_F^2=\bigTheta{N d\prod_{l=1}^{k-1} n_l \beta_l^2 }$
    w.p.\ at least $1-\sum_{l=1}^{k-1} \bigexp{-\bigOmg{n_l}} - \exp(-\bigOmg{d})$.

In the remainder of this section, we show how to lower bound $\lambda$.
First, we relate $\lambda$ to the smallest eigenvalue of (row-wise) Khatri-Rao powers of $F_{k-1}$. This is obtained via the following lemma, which is proved in Appendix \ref{app:lem:bound_lambda_star_k}.
\begin{lemma}\label{lem:bound_lambda_star_k}
    Fix any $k\in[L-1]$ and any integer $r>0.$ 
    Then, we have
    \begin{align*}
	&\evmin{ \E_{w\distas{}\mathcal{N}(0, \beta_{k+1}^2\, \Id_{n_k})} [\sigma(F_{k}w)\sigma(F_{k}w)^T] } \geq\beta_{k+1}^2\, \mu_r(\sigma)^2 \frac{\evmin{(F_{k}^{*r})(F_{k}^{*r})^T}}{\max_{i\in [N]}\norm{(F_{k})_{i:}}_2^{2(r-1)}} .
    \end{align*}
\end{lemma}
\revision{Next, we relate the Khatri-Rao powers of $F_k$ to a certain matrix involving the centered features $\tilde{F}_k=F_k-\E_X[F_k].$} This is formalized in the following lemma, which is proved in Appendix \ref{app:lem:Fk_vs_Fk_tilde}.

\begin{lemma}[Centering features]\label{lem:Fk_vs_Fk_tilde}
    Fix any $k\in[L-1]$, and any integer $r>0$. \revision{Let $\mu=\E_x[f_k(x)]\in\RR^{n_{k}}$ and $\Lambda = \diag(F_k\mu-\|\mu\|_2^2 1_N),$ where $1_N\in\RR^N$ is the all-one vector. 
    Then, we have 
    \begin{equation}
	(F_{k}^{*r}) (F_{k}^{*r})^T=(F_kF_k^T)^{\circ r} \succeq \left(\tilde F_k\tilde F_k^T-\frac{\Lambda 1_N 1_N^T\Lambda}{\|\mu\|_2^2}\right)^{\circ r},
    \end{equation}
    where $M^{\circ r}$ denotes the $r$-th Hadamard power of the matrix $M$.}
\end{lemma}
The last step is to bound the smallest eigenvalue of the \revision{matrix  $\left(\tilde F_k\tilde F_k^T-\frac{\Lambda 1_N 1_N^T\Lambda}{\|\mu\|_2^2}\right)^{\circ r},$} as done in the following lemma which is proved in Appendix \ref{app:lem:svmin_Fk_tilde_Khatri_Rao}. 
\begin{lemma}[\revision{Hadamard} powers of centered features]\label{lem:svmin_Fk_tilde_Khatri_Rao}
    Fix any $k\in[L-1]$ and any \revision{even integer $r\ge 2.$ }
    Assume $\prod_{l=1}^{k-1} \log(n_l)=\littleO{\min_{l\in[0,k]} n_l}.$
    Then, we have
    \begin{equation}\label{eq:dresnew}
	\evmin{\left(\tilde F_k\tilde F_k^T-\frac{\Lambda 1_N 1_N^T\Lambda}{\|\mu\|_2^2}\right)^{\circ r}}
	= \bigTheta{ \left( d \prod_{l=1}^{k} n_l\beta_l^2 \right)^r } 
    \end{equation}
    w.p.\ at least
    \begin{align}\label{eq:problemma}
	 &1 - N^2 \bigexp{ - \bigOmg{ \frac{\min_{l\in[0,k]}n_l}{N^{2/(\revision{r/2}-0.1)} \prod_{l=1}^{k-1}\log(n_l)} } } - N \sum_{l=1}^{k} \bigexp{-\bigOmg{n_l}} .
    \end{align}
\end{lemma}
    Combining these lemmas, one gets the desired lower bound of $\svmin{F_k}^2$.
    For the upper bound: $\evmin{F_kF_k^T}\leq \min_{i\in[N]}\norm{(F_k)_{i:}}_2^2 = \bigO{d\prod_{l=1}^{k} n_l\beta_l^2},$
    where we use Lemma \ref{lem:norm_feature_map} in Appendix \ref{app:more_results}.

\section{Lipschitz Constant of Feature Maps}\label{sec:lip_const_fk}
The Lipschitz constants of the feature maps $g_k:\RR^d\to\RR^{n_k}$ are critical to several proofs of this paper,
including Lemma \ref{lem:Fk_vs_Fk_tilde} and Lemma \ref{lem:svmin_Fk_tilde_Khatri_Rao}.
A simple upper bound is given by
    $\norm{g_k}_{\Lip}
    \leq\prod_{l=1}^{k} \norm{W_l}_{\op}$. 
    From standard bounds on the operator norm of Gaussian matrices (see \GaussOpNorm),
    one obtains that 
$\prod_{l=1}^{k} \norm{W_l}_{\op}$
 scales as $\prod_{l=1}^{k}\beta_l\max(\sqrt{n_{l-1}},\sqrt{n_l}).$
However, this simple estimate leads to restrictions on the network architectures 
for which our Theorem \ref{thm:empirical_Jacobian} holds. 
The product of many large random matrices is also studied in \citep{hanin2019products}, where it is shown that the logarithm of the $\ell_2$ norm between the Jacobian of deep networks and any fixed vector is asymptotically Gaussian. However, the findings of \citep{hanin2019products} are not applicable to our setting, which would require bounds that hold with probability exponentially close to $1$.



As usual, let $(W_l)_{ij}\distas{}\mathcal{N}(0,\beta_l^2)$ for $l\in[L].$
For every $z\in\RR^d$, denote its activation pattern up to layer $k$ by
$$\mathcal{A}_{1\to k}(z)=\left[\sign(g_{lj}(z))\right]_{l\in[k],j\in[n_l]}\in\Set{-1,0,1}^{\sum_{l=1}^{k}n_l},$$
where $\sign(g_{lj}(z))=1$ if $g_{lj}(z)>0$, $-1$ if $g_{lj}(z)<0$ and $0$ otherwise.
For every differentiable point of $g_k$, we denote by $J(g_k)(z)\in\RR^{n_k\times d}$ the corresponding Jacobian matrix.

Our starting point is to relate the Lipschitz constant of $g_k$ with the operator norm of its Jacobian.
First, we have via the Rademacher theorem that $\norm{g_k}_{\Lip}=\sup_{z\in\RR^d\setminus\Omega_{g_k}}\norm{J(g_k)(z)}_{\op}$,
where $\Omega_{g_k}$ is the set of non-differentiable points of $g_k$ which has measure zero.
The issue here is that even if we restrict ourself to the ``good'' set $\RR^d\setminus\Omega_{g_k}$,
the formula of the Jacobian matrix as computed by the standard back-propagation algorithm\footnote{using a convention that $\sigma'(0)=0$} (which is also the object that we know how to handle analytically) may not represent the true Jacobian of $g_k.$
This happens, for example, when the input to any of the ReLU activations is $0$. 
The following lemma circumvents this problem by restricting the supremum to the set of inputs where the two Jacobian matrices agree. 
Its proof is deferred to Appendix \ref{app:lem:lip_const_fk}.
\begin{lemma}\label{lem:lip_const_fk}
    Fix any $k\in[L].$ 
    Then w.p.\ 1 over $(W_l)_{l=1}^{k-1}$, the following holds for all choices of $W_k$:
    \begin{align}\label{eq:lip_prop}
	\norm{g_k}_{\Lip} 
	= \max\limits_{z\in\RR^d:\; \mathcal{A}_{1\to k-1}(z)\in\Set{-1,+1}^{\sum_{l=1}^{k-1}n_l}}\; \norm{J(g_k)(z)}_{\op}.
    \end{align}
\end{lemma}
In words, Lemma \ref{lem:lip_const_fk} shows that the Lipschitz constant of $g_k$
is given by the maximum operator norm of its Jacobian over all the inputs $z$'s which fulfill $g_{lj}(z)\neq 0$ for all $l\in[k-1],j\in[n_l].$
This has two implications. First, $g_k$ is differentiable at every such input, 
and chain rules can be applied through all the layers to compute the true Jacobian.
In particular, we have for all such $z$'s that:
\begin{align}\label{eq:Jab_form}
    J(g_k)(z)=W_k^T\prod_{l=1}^{k-1}\Sigma_{k-l}(z) W_{k-l}^T.
\end{align}
Second, one observes that $J(g_k)(z)=J(g_k)(z')$ for all $z,z'$ with $\mathcal{A}_{1\to k-1}(z)=\mathcal{A}_{1\to k-1}(z').$
Thus, the number of Jacobian matrices that one needs to bound in \eqref{eq:lip_prop}
is at most the number of activation patterns, which has been studied in \citep{HaninRolnick2019,MontufarEtal2014,SerraEtal2018}.
By exploiting these facts via a careful induction argument, we obtain the following result.
\begin{theorem}[Lipschitz constant of feature maps]\label{thm:lip_const_fk}
    Fix any $k\in[L-1].$
    Then, we have w.p.\ at least $1-\sum_{l=1}^{k}\bigexp{-\bigOmg{n_l}}$ 
    that
   \begin{align}\label{eq:bound_lip_gk}
	\norm{g_k}_{\Lip}^2
	=\bigO{ \frac{\prod_{l=0}^k n_l}{\min_{l\in[0,k]}n_l}\, \prod_{l=1}^{k-1}\log(n_l)\, \prod_{l=1}^{k}\beta_l^2} .
    \end{align}
\end{theorem}
The idea of the proof is to bound the operator norm of the Jacobian matrix from \eqref{eq:Jab_form} for all inputs having a given activation pattern 
(via an $\epsilon$-net argument and concentration inequalities),
and then to do a union bound over all the possible patterns. The details are deferred to Appendix \ref{app:thm:lip_const_fk}. 


%


\section{Further Related Work}\label{sec:related_work}
The spectrum of various random matrices arising from deep learning models has been the subject of recent investigations.
Most of the existing results focus on the linear-width asymptotic regime,
where the widths of the various layers are linearly proportional. 
In particular, the spectrum of the conjugate kernel (CK) is studied in the single-layer case for Gaussian i.i.d.\ data \citep{pennington2017nonlinear}, 
for Gaussian mixtures \citep{liao2018spectrum}, for general training data \citep{louart2018random}, 
and for a model with an additive bias \citep{adlam2019random}.
The multi-layer case is tackled in \citep{benigni2019eigenvalue}. 
The Hessian matrix of a two-layer network can be decomposed into two pieces, 
one coming from the second derivatives and the other of the form $J^TJ$ (a.k.a. the Fisher information matrix).
This second term is studied in \citep{pennington2017geometry,pennington2018spectrum}. 
Note that this is different from the NTK matrix, given by $JJ^T$, as analyzed in this paper.
Typically, for an over-parameterized model, the Fisher information matrix is rank-deficient, whereas the NTK one is full-rank.
The work \citep{pennington2018emergence} uses tools from free probability to study the spectrum of the input-output Jacobian of the network.
Again, this is different from the parameter-output Jacobian considered in this paper.
Generalization error has been also studied via the spectrum of suitable random matrices: for linear regression \citep{hastie2019surprises}, 
random feature models \citep{mei2019generalization}, random Fourier features \citep{liao2020random}, 
and most recently for a two-layer network \citep{Andrea2020}.
    
Generally speaking, the line of literature reviewed above has studied the spectrum of various random matrices 
related to neural networks. Our work is complementary
in the sense that it concerns the smallest eigenvalue of the NTK and the feature maps. 
We remark that obtaining an almost-sure convergence of the empirical spectral distribution of a random matrix in general
does not have any implications on the limit of its individual eigenvalues. 
The closest existing work is \citep{Andrea2020}, 
which focuses on a two-layer model and gives a lower bound on the smallest eigenvalue of the NTK matrix when the number of parameters of the network exceeds the number of training samples.

\section{Conclusions and Open Problems}

This paper provides tight bounds on the smallest eigenvalues of NTK matrices for deep ReLU networks. In the finite-width setting, our result holds for networks
with a single wide layer, regardless of its position, as long as the wide layer has roughly order of $N$ neurons. This gives hope that gradient descent methods will be successful in optimizing such architectures.
However, we note that it is not possible to directly apply existing results in the literature such as \citep{ChizatEtc2019},
since the Jacobian matrix is not Lipschitz with respect to the weights. Furthermore, to get optimization guarantees, one often has to track the movement of the NTK-related quantities during the course of training,
which is not done in this paper. Providing rigorous convergence guarantees for deep ReLU networks 
with an {\em arbitrary} single wide layer of linear width is an exciting open problem. Other interesting extensions include the study of networks with biases and non-Gaussian initializations.


 \section*{Acknowledgements}
The authors would like to thank the anonymous reviewers for their helpful comments. MM was partially supported by the 2019 Lopez-Loreta Prize.
 QN and GM acknowledge support from the European Research Council (ERC)
 under the European Union’s Horizon 2020 research and innovation
 programme (grant agreement no 757983). \revision{MM would like to thank Simone Bombari, Adel Javanmard and Mahdi Soltanolkotabi for helpful discussions concerning the edit of Lemmas \ref{lem:Fk_vs_Fk_tilde} and \ref{lem:svmin_Fk_tilde_Khatri_Rao}.}
 
 \newpage

\bibliography{regul}
\bibliographystyle{icml2021}
 
\onecolumn 
\appendix 
\section{Additional Notations}\label{sec:app}
Given a sub-exponential random variable $X$, let
    $\|X\|_{\psi_1} = \inf \{ t>0 \,\,: \,\,\mathbb E[\exp(|X|/t)] \le 2 \}.$
Similarly, for a sub-gaussian random variable, 
    $\|X\|_{\psi_2} = \inf \{ t>0 \,\,: \,\,\mathbb E[\exp(X^2/t^2)] \le 2 \}.$

\section{Proof of Theorem \ref{thm:limiting_NTK}}\label{app:prooflim}
    Let us first get some useful estimates from the data.
    By Assumptions \ref{ass:data_dist} and \ref{ass:data_dist2}, 
    we have $\norm{x_i}_2^{2}=\Theta(d)$ for all $i\in[N]$ w.p.\ $\geq 1-Ne^{-\bigOmg{d}}.$
    For a given pair $i\neq j$, let $x_j$ be fixed and $x_i$ be random, then $\inner{x_i,x_j}$ is Lipschitz continuous w.r.t. $x_i$, 
    where the Lipschitz constant is given by 
    $\norm{x_j}_2=\mathcal{O}(\sqrt{d}).$
    Thus, it follows from Assumption \ref{ass:data_dist2} that $\PP(\abs{\inner{x_i,x_j}} > t)\leq 2e^{-t^2/\bigO{d}}.$
    By picking $t=d N^{-1/(r-0.5)}$ and doing a union bound over all data pairs, we get
    $\max_{i\neq j}\abs{\inner{x_i,x_j}}^r\leq d N^{-1/(r-0.5)}$ w.p.\ at least $1-N^2e^{-\bigOmg{dN^{-2/(r-0.5)}}}.$
    Combining these two events, we obtain that the following hold
    \begin{align}\label{eq:data_estimates}
	&\norm{x_i}_2^2=\Theta(d),\ \forall\,i\in[N],\nonumber \\
	&\abs{\inner{x_i,x_j}}^r\leq d N^{-1/(r-0.5)},\ \forall\,i\neq j 
    \end{align}
    with the same probability as stated in the theorem.

    We have from Lemma \ref{lem:limNTK_matform} that
    \begin{align*}
	K^{(L)} = \sum_{l=1}^{L} G^{(l)} \circ \dot{G}^{(l+1)}\circ\dot{G}^{(l+2)}\circ\ldots\circ\dot{G}^{(L)} .
    \end{align*}
    One also observes that all the matrices $G^{(l)},\dot{G}^{(l)},G^{(l)}$ are positive semidefinite. 
    Recall that, for two p.s.d.\ matrices $P,Q\in\RR^{n\times n}$, 
    one has $\evmin{P\circ Q}\geq\evmin{P}\min_{i\in[n]}Q_{ii}$ \citep{schur1911bemerkungen}.
    Thus, it holds
    \begin{align*}	
	\evmin{K^{(L)}} 
	\geq \sum_{l=1}^{L} \evmin{G^{(l)}} \min_{i\in[N]} \prod_{p=l+1}^{L} (\dot{G}^p)_{ii} 
	= \sum_{l=1}^{L} \evmin{G^{(l)}},
    \end{align*}
    where the last equality follows from the fact that $(\dot{G}^{(p)})_{ii}=1$ for all $p\in[2,L],i\in[N].$
    From here, it suffices to bound $\evmin{G^{(2)}}.$
    Let $D=\diag([\norm{x_i}_2]_{i=1}^{N})$ and $\hat{X}=D^{-1}X.$ 
    Then, by the homogeneity of $\sigma$, we have
    $\sigma(Xw)=\sigma(D\hat{X}w)=D\sigma(\hat{X}w)$, and thus 
    \begin{align}\label{eq:G2_bound}
	\evmin{ G^{(2)} } 
	&=\evmin{ D \E\left[\sigma(\hat{X}w)\sigma(\hat{X}w)^T\right] D } \nonumber\\
	&=\evmin{D \left[\mu_0(\sigma)^2 1_N 1_N^T + \sum_{s=1}^{\infty}\mu_s(\sigma)^2 (\hat{X}^{*s}) (\hat{X}^{*s})^T \right] D} \nonumber\\
	&\geq\mu_r(\sigma)^2 \evmin{D (\hat{X}^{*r}) (\hat{X}^{*r})^T D} \nonumber\\
	&=\mu_r(\sigma)^2 \evmin{D^{-(r-1)} (X^{*r}) (X^{*r})^T D^{-(r-1)}} \nonumber\\
	&\geq\mu_r(\sigma)^2 \frac{\evmin{(X^{*r})(X^{*r})^T}}{\max_{i\in [N]}\norm{x_i}_2^{2(r-1)}} ,
    \end{align}
    where the second step uses the Hermite expansion of $\sigma$ (for the proof see Lemma D.3 of \citep{QuynhMarco2020}).
    By Gershgorin circle theorem, one has
    \begin{align*}
	\evmin{(X^{*r})(X^{*r})^T} 
	\geq \min_{i\in[N]} \norm{x_i}_2^{2r} - (N-1) \max_{i\neq j} \abs{\inner{x_i,x_j}}^r
	\geq\Omega(d),
    \end{align*}    
    where the last estimate follows from \eqref{eq:data_estimates}.
    Plugging this and the estimate of \eqref{eq:data_estimates} into the inequality \eqref{eq:G2_bound} proves the lower bound on the smallest eigenvalue of the NTK.
    For the upper bound, note that
    \begin{align*}
	\evmin{K^{(L)}} 
	\leq \frac{\tr(K^{(L)})}{N}
	=\frac{1}{N}\sum_{i=1}^{N}\sum_{l=1}^{L} (G^{(l)})_{ii} \prod_{p=l+1}^{L} (\dot{G}^p)_{ii} .
    \end{align*}
    One observes that $(G^{(l)})_{ii}=2\E_{g\distas{}\mathcal{N}(0,(G_{l-1})_{ii})} [\sigma(g)^2]=(G^{(l-1)})_{ii}.$
    Iterating this argument gives $(G^{(l)})_{ii}=(G^{(1)})_{ii}=\norm{x_i}_2^2.$
    Thus, it follows that
    \begin{align*}
	\evmin{K^{(L)}} 
	\leq \frac{L}{N}\tr(G^{(1)})
	= \frac{L}{N}\sum_{i=1}^{N} \norm{x_i}_2^2 = L\,\mathcal{O}(d),
    \end{align*}
    where we used again \eqref{eq:data_estimates} in the last estimate.

\section{Some Useful Estimates}\label{app:more_results}
\begin{lemma}\label{lem:norm_feature_map}
    Fix any $0\leq k\leq L-1$ and $x\distas{}P_X.$
    Then, we have
    \begin{align*}
	    \norm{f_k(x)}_2^2 = \bigTheta{ d\prod_{l=1}^k n_l \beta_l^2 } 
    \end{align*}
    w.p.\ at least $1-\sum_{l=1}^{k} \bigexp{-\bigOmg{n_l}} - \exp(-\bigOmg{d})$ over $(W_l)_{l=1}^{k}$ and $x.$
    Moreover, 
    \begin{align*}
	\E_x \norm{f_k(x)}_2^2 = \bigTheta{ d\prod_{l=1}^k n_l \beta_l^2 }
    \end{align*}
    w.p.\ $1-\sum_{l=1}^{k} \bigexp{-\bigOmg{n_l}}$ over $(W_l)_{l=1}^{k}.$
\end{lemma}

\begin{lemma}\label{lem:norm_of_expected_feature_map}
    Fix any $k\in[L-1].$
    Then, we have
    \begin{align*}
	    \norm{\E_x[f_k(x)]}_2^2 = \bigTheta{ d\prod_{l=1}^{k} n_l \beta_l^2 } 
    \end{align*}
    w.p.\ at least $1 - \sum_{l=1}^k\bigexp{-\bigOmg{n_l}}$ over $(W_l)_{l=1}^k.$
\end{lemma}

\begin{lemma}\label{lem:square_fx}
    Fix any $k\in[L-1].$ 
    Assume $\prod_{l=1}^{k-1} \log(n_l)=\littleO{\min_{l\in[0,k]} n_l}.$ 
    Then, we have
    \begin{align}
	    \norm{f_{k}(x_i) - \E_{x}[f_{k}(x)]}_2^2 = \bigTheta{ d \prod_{l=1}^{k} n_l\beta_l^2} ,\quad\forall\,i\in[N] 
    \end{align}
    w.p.\ at least
    \begin{align*}
	1-N\bigexp{ -\bigOmg{ \frac{\min_{l\in[0,k]} n_l}{\prod_{l=1}^{k-1}\log(n_l)} } }
	  - \sum_{l=1}^{k}\exp(-\bigOmg{n_l}) .
    \end{align*}
\end{lemma}

\begin{lemma}\label{lem:Ex_square_fx}
    Fix any $k\in[L-1].$ Then, we have
    \begin{align*}
	    \E_x \norm{f_k(x) - \E_{x}[f_k(x)]}_2^2 = \bigTheta{ d \prod_{l=1}^k n_l\beta_l^2 }
    \end{align*}
    w.p.\ at least $1-\sum_{l=1}^k \bigexp{-\bigOmg{n_l}}$ over $(W_l)_{l=1}^k.$
\end{lemma}

\begin{lemma}\label{lem:frobnorm_SIGMA}
    Fix any $k\in[L-1]$, and $x\distas{}P_X.$ 
    Then, we have that $\norm{\Sigma_k(x)}_F^2=\bigTheta{n_k}$ w.p.\ at least $1-\sum_{l=1}^{k} \bigexp{-\bigOmg{n_l}} - \exp(-\bigOmg{d})$ over $(W_l)_{l=1}^{k}$ and $x.$
\end{lemma}

\begin{lemma}\label{lem:frobnorm_W_SIGMA}
    Fix any $k\in[L-1], k\leq p\leq L-1$, and $x\distas{}P_X.$
    Then, we have that
    \begin{align*}
	    \norm{\Sigma_{k}(x)\prod_{l=k+1}^{p} W_l\Sigma_l(x) }_F^2 
	    = \bigTheta{n_{k} \prod_{l=k+1}^{p} n_l\beta_l^2}
    \end{align*}
    w.p.\ at least
	$1 - \sum_{l=1}^{p} \bigexp{-\bigOmg{n_l}} - \exp(-\bigOmg{d})$
    over $(W_l)_{l=1}^{p}$ and $x.$
\end{lemma}

\subsection{Proof of Lemma \ref{lem:norm_feature_map}}
    The proof works by induction over $k.$
    Note that the statement holds for $k=0$ due to Assumptions \ref{ass:data_dist} and \ref{ass:data_dist2}.
    Assume that the lemma holds for some $k-1$,
    i.e. $\norm{f_{k-1}(x)}_2^2=\bigTheta{d\prod_{l=1}^{k-1}n_l\beta_l^2}$ w.p.\ at least $ 1-\sum_{l=1}^{k-1} N\bigexp{-\bigOmg{n_l}}-N\bigexp{-\bigOmg{d}}.$
    Let us condition on this event of $(W_l)_{l=1}^{k-1}$ and study probability bounds over $W_k.$
    Let $W_k=[w_1,\ldots,w_{n_k}]^T$ where $w_j\distas{}\mathcal{N}(0, \beta_k^2\,\Id_{n_{k-1}}).$
    Note that
    \begin{align}\label{eq:lempf1}
	\norm{f_k(x)}_2^2=\sum_{j=1}^{n_k} f_{k,j}(x)^2,
    \end{align}
and that
    \begin{align*}
	\E_{W_k} \norm{f_k(x)}_2^2
	= \sum_{j=1}^{n_k} \E_{w_j}[f_{k,j}(x)^2]
	= \frac{n_k\beta_k^2}{2}\norm{f_{k-1}(x)}_2^2
	= \bigTheta{d\prod_{l=1}^{k}n_l\beta_l^2} ,
    \end{align*}
where the last equality follows from the induction assumption. Furthermore,
    \begin{align*}
	\norm{f_{k,j}(x)^2}_{\psi_1}
	= \norm{f_{k,j}(x)}_{\psi_2}^2 
	\le c \beta_k^2\, \norm{f_{k-1}(x)}_2^2 
	=\bigO{\beta_k^2 d\prod_{l=1}^{k-1}n_l\beta_l^2} ,
    \end{align*}
    where $c$ is an absolute constant.
    Thus, by applying Bernstein's inequality (see \Bernstein) to the sum of i.i.d.\ random variables in \eqref{eq:lempf1}, we have
    \begin{align*}
	\frac{1}{2} \E_{W_k} \norm{f_k(x)}_2^2 \leq \norm{f_k(x)}_2^2 \leq \frac{3}{2} \E_{W_k} \norm{f_k(x)}_2^2
    \end{align*}
    w.p.\ at least $1-\bigexp{-\bigOmg{n_k}}.$ 
    Taking the intersection of the two events finishes the proof for $\norm{f_k(x)}_2^2.$
    The proof for $\E_x\norm{f_k(x)}_2^2$ can be done by following similar passages and using that $\norm{\E_x[f_{k,j}(x)^2]}_{\psi_1}\leq\E_x\norm{f_{k,j}(x)^2}_{\psi_1}$.

\subsection{Proof of Lemma \ref{lem:norm_of_expected_feature_map}}
    The upper bound follows from Lemma \ref{lem:norm_feature_map} via Jensen's inequality.
    The proof for the lower bound works by induction on $k.$
    Assume it holds for $k-1$ that $\norm{\E_x[f_{k-1}(x)]}_2^2 = \bigOmg{ d\prod_{l=1}^{k-1} n_l \beta_l^2 } $
    w.p.\ at least $1 - \sum_{l=1}^{k-1}\bigexp{-\bigOmg{n_l}}$ over $(W_l)_{l=1}^{k-1}.$
    Let us condition on the intersection of this event and that of Lemma \ref{lem:norm_feature_map} for $(W_l)_{l=1}^{k-1}.$
    Let $W_k=[w_1,\ldots,w_{n_k}]$ where $w_j\distas{}\mathcal{N}(0, \beta_k^2\,\Id_{n_{k-1}}).$
    For every $j\in[n_k],$
    \begin{align*}
	\norm{(\E_x[f_{k,j}(x)])^2}_{\psi_1}
	= \norm{\E_x[f_{k,j}(x)]}_{\psi_2}^2
	\leq \E_x\norm{[f_{k,j}(x)]}_{\psi_2}^2
	\le c \beta_k^2\, \E_x\norm{f_{k-1}(x)}_2^2 
	= \bigO{ d \beta_k^2\, \prod_{l=1}^{k-1} n_l \beta_l^2 },
    \end{align*}
    where $c$ is an absolute constant and the last equality follows from the above conditional event from Lemma \ref{lem:norm_feature_map}.
    Moreover,
    \begin{align*}
	\E_{W_k} \norm{\E_x[f_k(x)]}_2^2
	&=\sum_{j=1}^{n_k} \E_{w_j} (\E_x[f_{k,j}(x)])^2 
	\geq \sum_{j=1}^{n_k} (\E_x\E_{w_j}[f_{k,j}(x)])^2 
	= \frac{n_k\beta_k^2}{2\pi}\, (\E_x\norm{f_{k-1}(x)})^2\\ 
	&\geq \frac{n_k\beta_k^2}{2\pi}\, \norm{\E_x[f_{k-1}(x)]}_2^2 
	= \bigOmg{ d\prod_{l=1}^k n_l \beta_l^2 } ,
    \end{align*}
    where the last estimate follows from our induction assumption.
    By Bernstein's inequality (see \Bernstein), we have
    \begin{align*}
	\norm{\E_x[f_k(x)]}_2^2 \geq \frac{1}{2} \E_{W_k} \norm{\E_x[f_k(x)]}_2^2 = \bigOmg{ d\prod_{l=1}^k n_l \beta_l^2 }
    \end{align*}
    w.p.\ at least $1-\bigexp{-n_k}$ over $W_k.$
    Taking the intersection of all these events finishes the proof.

\subsection{Proof of Lemma \ref{lem:square_fx}}
    Let $Z:\RR^d\to\RR$ be a random function over $x_i$ defined as $Z(x_i)=\norm{f_{k}(x_i) - \E_{x}[f_{k}(x)]}_2.$
    It follows from Theorem \ref{thm:lip_const_fk} that w.p.\ at least $1-\sum_{l=1}^{k}\bigexp{-\bigOmg{n_l}}$ over $(W_l)_{l=1}^{k},$
\begin{equation}\label{eq:l1}
    \begin{split}
	\norm{Z}_{\Lip}^2
	=\bigO{ \frac{\prod_{l=0}^k n_l}{\min_{l\in[0,k]}n_l}\, \prod_{l=1}^{k-1}\log(n_l)\, \prod_{l=1}^{k}\beta_l^2 } 
	=\littleO{d\prod_{l=1}^{k}n_l\beta_l^2}.
    \end{split}
\end{equation}  
    Below, let us denote the shorthand 
    \begin{align*}
	\E[Z]=\E_{x_i}[Z(x_i)]=\int_{\RR^d} Z(x_i) dP_X(x_i) .
    \end{align*}
    It holds
\begin{equation}\label{eq:l2}
    \begin{split}
	\E [Z]^2 
	&= \E [Z^2] - \E[\abs{Z-\E Z}^2] \\
	&\geq \E [Z^2] - \int_0^\infty \PP(\abs{Z-\E Z}>\sqrt{t}) dt \\
	&\geq \E [Z^2] - \int_0^\infty 2\bigexp{-\frac{c\,t}{\norm{Z}_{\Lip}^2}}  dt \\
	&= \E [Z^2] - \frac{2}{c}\norm{Z}_{\Lip}^2,
    \end{split}
\end{equation}  
    where the 2nd inequality follows from Assumption \ref{ass:data_dist2}. 
    By Lemma \ref{lem:Ex_square_fx}, we have w.p.\ at least $1-\sum_{l=1}^k \bigexp{-\bigOmg{n_l}}$ over $(W_l)_{l=1}^{k}$ that
\begin{equation}\label{eq:l3}
    \begin{split}
	    \E[Z^2]=\bigTheta{d\prod_{l=1}^{k}n_l\beta_l^2} .
    \end{split}
\end{equation}  
 By combining \eqref{eq:l1}, \eqref{eq:l2} and \eqref{eq:l3}, we obtain that $\E[Z]=\bigOmg{\sqrt{d\prod_{l=1}^{k}n_l\beta_l^2}}.$
    Moreover, 
	$\E[Z]\leq\sqrt{\E[Z^2]}=\bigO{\sqrt{d\prod_{l=1}^{k}n_l\beta_l^2}} .$
    As a result, we have that $\E[Z]=\bigTheta{\sqrt{d\prod_{l=1}^{k}n_l\beta_l^2}}$ w.p.\ at least $1 - \sum_{l=1}^{k}\exp(-\bigOmg{n_l})$ over $(W_l)_{l=1}^{k}.$
    Let us condition on this event and study probability bounds over the samples.
    Using Assumption \ref{ass:data_dist2},
    we have $\frac{1}{2}\E[Z]\leq Z\leq\frac{3}{2}\E[Z],$ hence $Z=\bigTheta{\sqrt{d\prod_{l=1}^{k}n_l\beta_l^2}},$
    w.p.\ at least
    \begin{align*}
	1-\bigexp{ -\bigOmg{ \frac{\min_{l\in[0,k]} n_l}{\prod_{l=1}^{k-1}\log(n_l)} } } .
    \end{align*}
    Taking the union bound over $N$ samples, followed by an intersection with the above event over the weights, finishes the proof.

\subsection{Proof of Lemma \ref{lem:Ex_square_fx}}
    The proof works by induction on $k.$
    Note that the statement holds for $k=0$ due to Assumption \ref{ass:data_dist}.
    Let us assume for now that the result holds for the first $k$ layers.
    To prove it for layer $k$, we condition on the intersection of this event and the event of Lemma \ref{lem:norm_feature_map} for $(W_l)_{l=1}^{k-1}$,
    and study probability bounds over $W_k$.
    Define $W_k=[w_1,\ldots,w_{n_k}]\in\RR^{n_{k-1}\times n_k}$ where $w_j\distas{}\mathcal{N}(0,\beta_k^2\Id_{n_{k-1}}).$
    Recall that by definition, $f_{k,j}(x)=\sigma(\inner{w_j, f_{k-1}(x)})$ for $j\in[n_k].$
    We have that
    \begin{align*}
	\E_x \norm{f_k(x) - \E_{x}[f_k(x)]}_2^2 = \sum_{j=1}^{n_k} \E_x \Big(f_{k,j}(x) - \E_{x}[f_{k,j}(x)]\Big)_2^2 .
    \end{align*}
    Taking the expectation over $W_k$, we have
    \begin{align*}
	&\E_{W_k} \E_x \norm{f_k(x) - \E_{x}[f_k(x)]}_2^2\\
	&= \E_{W_k}\E_x\norm{f_k(x)}_2^2 - \E_{W_k}\norm{\E_x[f_k(x)]}_2^2 \\
	&= \frac{n_k\beta_k^2}{2}\, \E_x\norm{f_{k-1}(x)}_2^2 - \E_x\E_y \sum_{j=1}^{n_k} \E_{w_j} \sigma\left(\inner{w_j,f_{k-1}(x)}\right) \sigma\left(\inner{w_j,f_{k-1}(y)}\right) \\
	&= \frac{n_k\beta_k^2}{2}\, \E_x\norm{f_{k-1}(x)}_2^2 - n_k\beta_k^2\, \E_x\E_y\norm{f_{k-1}(x)}_2 \norm{f_{k-1}(y)}_2  \sum_{r=0}^{\infty} \mu_r(\sigma)^2 \inner{\frac{f_{k-1}(x)}{\norm{f_{k-1}(x)}_2},\frac{f_{k-1}(y)}{\norm{f_{k-1}(y)}_2}}^r  \\
	&\geq \frac{n_k\beta_k^2}{2}\, \E_x\norm{f_{k-1}(x)}_2^2 - \mu_1(\sigma)^2\,n_k\beta_k^2\, \norm{\E_x[f_{k-1}(x)]}_2^2 - n_k\beta_k^2\, \sum_{\substack{r=0\\r\neq 1}}^{\infty} \mu_r(\sigma)^2 (\E_x\norm{f_{k-1}(x)})^2 \\
	&= \frac{n_k\beta_k^2}{2}\, \E_x\norm{f_{k-1}(x)}_2^2 - \frac{n_k\beta_k^2}{4}\, \norm{\E_x[f_{k-1}(x)]}_2^2 - \frac{n_k\beta_k^2}{4}\, (\E_x\norm{f_{k-1}(x)})^2,
    \end{align*}
where in the last step we use that $\mu_1(\sigma)^2=1/4$ and that $\sum\limits_{\substack{r=0\\r\neq 1}}^{\infty}	\mu_r(\sigma)^2=1/4$. 
Furthermore, the RHS of the last expression can be lower bounded by
    \begin{align*}
	 \frac{n_k\beta_k^2}{4}\, \left( \E_x\norm{f_{k-1}(x)}_2^2 - \norm{\E_x[f_{k-1}(x)]}_2^2 \right) 
	= \frac{n_k\beta_k^2}{4}\, \E_x\norm{f_{k-1}(x)-\E_x[f_{k-1}(x)]}_2^2 
	=\bigOmg{ d \prod_{l=1}^k n_l\beta_l^2 } ,
    \end{align*}
    where the last step follows by induction assumption.
    Moreover, it follows from above that
    \begin{align*}
	\E_{W_k} \E_x \norm{f_k(x) - \E_{x}[f_k(x)]}_2^2
	\leq \frac{n_k\beta_k^2}{2}\, \E_x\norm{f_{k-1}(x)}_2^2
	= \bigO{ d \prod_{l=1}^k n_l\beta_l^2 } ,
    \end{align*}
    where the last estimate follows from Lemma \ref{lem:norm_feature_map}.
    For every $j\in[n_k],$
    \begin{align*}
	\norm{\E_x \Big(f_{k,j}(x) - \E_{x}[f_{k,j}(x)]\Big)^2}_{\psi_1}
	&\leq \E_x \norm{\Big(f_{k,j}(x) - \E_{x}[f_{k,j}(x)]\Big)^2}_{\psi_1} \\
	&= \E_x \norm{f_{k,j}(x) - \E_{x}[f_{k,j}(x)]}_{\psi_2}^2 \\
	&\le c\, \E_x \norm{f_{k,j}(x)}_{\psi_2}^2 \\
	&\le c \,\E_x \left( \norm{f_{k,j}(x) - \E_{w_j}[f_{k,j}(x)]}_{\psi_2}^2 + \abs{\E_{w_j}[f_{k,j}(x)]}^2 \right)\\
	&\le c \,\E_x \left( \beta_k^2 \norm{f_{k,j}(x)}_{\Lip}^2 + \frac{\beta_k^2}{2\pi} \norm{f_{k-1}(x)}_2^2 \right)\\
	&\le c \beta_k^2\, \E_x \norm{f_{k-1}(x)}_2^2\\
	&= \bigO{ \beta_k^2 d \prod_{l=1}^{k-1} \beta_l^2 n_l },
    \end{align*}
    where $c$ is an absolute constant (which is allowed to change from line to line) and the last step uses Lemma \ref{lem:norm_feature_map}.
    By Bernstein's inequality (see \Bernstein), 
    \begin{align*}
	\frac{1}{2} \E_{W_k} \E_x \norm{f_k(x) - \E_{x}[f_k(x)]}_2^2\le \E_x \norm{f_k(x) - \E_{x}[f_k(x)]}_2^2\le \frac{3}{2} \E_{W_k} \E_x \norm{f_k(x) - \E_{x}[f_k(x)]}_2^2,
    \end{align*}
w.p.\ at least $1-\bigexp{-\bigOmg{n_k}}$ over $W_k.$ Thus, with that probability, we have that     \begin{align*}
	\E_x \norm{f_k(x) - \E_{x}[f_k(x)]}_2^2 
	=\bigTheta{ d \prod_{l=1}^k n_l\beta_l^2 }. 
    \end{align*}
    Taking the intersection of all the events finishes the proof.

\subsection{Proof of Lemma \ref{lem:frobnorm_SIGMA}}
\begin{proof}
    By Lemma \ref{lem:norm_feature_map}, 
    we have $f_{k-1}(x)\neq 0$ w.p.\ at least $1-\sum_{l=1}^{k-1} \bigexp{-\bigOmg{n_l}} - \exp(-\bigOmg{d})$ over $(W_l)_{l=1}^{k-1}$ and $x.$
    Let us condition on this event and derive probability bounds over $W_k.$
    Let $W_k=[w_1,\ldots,w_{n_k}].$
    Then, $\norm{\Sigma_k(x)}_F^2=\sum_{j=1}^{n_k} \sigma'(\inner{f_{k-1}(x), w_j}).$
    Thus,
    \begin{align*}
	\E_{W_k} \norm{\Sigma_k(x)}_F^2 
	=n_k \E_{w_1} [\sigma'(-\inner{f_{k-1}(x), w_1}))]
	=n_k \E_{w_1} [(1-\sigma'(\inner{f_{k-1}(x), w_1}))]
	=n_k-\E_{W_k} \norm{\Sigma_k(x)}_F^2, 
    \end{align*}
    where we used the fact that $w_j$ has a symmetric distribution, $\sigma'(t)=1-\sigma'(-t)$ for $t\neq 0$, 
    and the set of $w_1\in\RR^{n_{k-1}}$ for which $\inner{f_{k-1}(x), w_j}=0$ has measure zero.
    This implies that $\E_{W_k}\norm{\Sigma_k(x)}_F^2=n_k/2.$
    By Hoeffding's inequality on bounded random variables (see \Hoeff), we have  
    \begin{align*}
	\PP\left( \abs{\norm{\Sigma_k(x)}_F^2 - \E_{W_k}\norm{\Sigma_k(x)}_F^2 }>t \right) \leq 2\bigexp{ -\frac{2t^2}{n_k} } . 
    \end{align*}
    Picking $t=n_k/4$ finishes the proof.
\end{proof}

\subsection{Proof of Lemma \ref{lem:frobnorm_W_SIGMA}}
    The proof works by induction on $p.$
    First, Lemma \ref{lem:frobnorm_SIGMA} implies that the statement holds for $p=k.$
    Suppose it holds for some $p-1.$
    Note that this implies $f_{p-1}(x)\neq 0$ because otherwise $\Sigma_{p-1}(x)=0$, which contradicts the induction assumption.
    Let $S_p=\Sigma_{k}(x)\prod_{l=k+1}^{p} W_l\Sigma_l(x).$
    Then, $S_p=S_{p-1} W_p\Sigma_p(x).$
    Let $W_p=[w_1,\ldots,w_{n_p}].$
    Then,
    \begin{align*}
	\norm{S_p}_F^2 
	= \sum_{j=1}^{n_p} \norm{S_{p-1}w_j}_2^2 \sigma'(g_{p,j}(x))
	= \sum_{j=1}^{n_p} \norm{S_{p-1}w_j}_2^2 \sigma'(\inner{f_{p-1}(x),w_j}).
    \end{align*}
    We have 
    \begin{align*}
	\E_{W_p} \norm{S_p}_F^2
	&= n_p \E_{w_1} \norm{S_{p-1}w_1}_2^2 \sigma'(\inner{f_{p-1}(x),w_1}) \\
	&= n_p \E_{w_1} \norm{S_{p-1}(-w_1)}_2^2 \sigma'(\inner{f_{p-1}(x),(-w_1)}) \\
	&= n_p \E_{w_1} \norm{S_{p-1}w_1}_2^2 (1-\sigma'(\inner{f_{p-1}(x),w_1})) \\
	&= n_p \E_{w_1} \norm{S_{p-1}w_1}_2^2 - \E_{W_p} \norm{S_p}_F^2 \\
	&= n_p \beta_p^2 \norm{S_{p-1}}_F^2 - \E_{W_p} \norm{S_p}_F^2 ,
    \end{align*}
    where the second step uses that $w_1$ has a symmetric distribution, the third step uses the fact that $\sigma'(t)=1-\sigma'(-t)$ for $t\neq 0$
    and the set of $w_1$ for which $\inner{f_{p-1}(x),w_1})=0$ has measure zero.
    Thus,
    \begin{align*}
	\E_{W_p} \norm{S_p}_F^2 = \frac{n_p}{2} \beta_p^2 \norm{S_{p-1}}_F^2 
	= \bigTheta{n_{k} \prod_{l=k+1}^{p} n_l\beta_l^2},
    \end{align*}
    where the last equality holds by induction assumption.
    Moreover,
    \begin{align*}
	\norm{ \norm{S_{p-1}w_j}_2^2 \sigma'(\inner{f_{p-1}(x),w_j}) }_{\psi_1}
	\le c \norm{ \norm{S_{p-1}w_j}_2 }_{\psi_2}^2 
	\le c \beta_p^2 \norm{S_{p-1}}_F^2 ,
    \end{align*}
    where $c$ is an absolute constant (which is allowed to change from passage to passage).
    By Bernstein's inequality (see \Bernstein), we have 
     \begin{align*}
	\frac{1}{2} \E_{W_p}\norm{S_p}_F^2 \leq \norm{S_p}_F^2 \leq \frac{3}{2} \E_{W_p}\norm{S_p}_F^2
     \end{align*}
    w.p.\ at least $1-e^{-\bigOmg{n_p}}.$
    Taking the intersection of all the events finishes the proof.

\section{Missing Proofs from Section \ref{sec:NTK_finite}}
\subsection{Proof of Corollary \ref{cor:zero_loss}}\label{app:zero_loss}
    Let $p=\sum_{l=1}^{L}n_ln_{l-1}$.
    Let $\frac{\partial F_L}{\partial\theta}\in\RR^{N\times p}$ denote the true Jacobian of $F_L$ (without the convention that $\sigma'(0)=0$) at a differentiable point $\theta.$
    Note that, by Lemma B.2 of \citep{QuynhMarco2020}, $F_L(\theta)$ is locally Lipschitz, thus a.e. differentiable.
    Let $J(\theta)\in\RR^{N\times p}$ be the Jacobian matrix defined in \eqref{eq:Jac} (with the convention that $\sigma'(0)=0$).
    Let $$\Omega_1=\Setbar{\theta\in\RR^p}{\rank(J(\theta))=N}$$ and $$\Omega_0=\Setbar{\theta\in\RR^p}{\exists l\in[L-1],j\in[n_l],i\in[N]: g_{lj}(x_i)=0} .$$ 
    Let $\lambda_p$ denote the Lebesgue measure in $\RR^p.$
    Pick an even integer $r$ s.t. $r\ge 0.1+2/\delta'$. Then, Theorem \ref{thm:empirical_Jacobian} implies that,
    with high probability (as stated in the corollary) over the training data, we have $\lambda_p(\Omega_1)>0.$ 
    For every $\theta\in\Omega_1$, it holds that $f_l(\theta,x_i)\neq 0$ for all $0\leq l\leq L-2,i\in[N]$,
    because otherwise $J(\theta)_{i:}=0$ (which leads to a contradiction).
    Thus, every $\theta\in\Omega_1\cap\Omega_0$ must satisfy $0=g_{lj}(\theta,x_i)=\inner{f_{l-1}(\theta,x_i), (W_l)_{:j}}$
    for some $l\in[L-1],j\in[n_l],i\in[N].$ 
    The set of $W_l$ which satisfies this equation has measure zero,
    and thus it holds $\lambda_p(\Omega_1\cap\Omega_0)=0.$
    Combining these facts, we get $\lambda_p(\Omega_1\setminus\Omega_0)>0.$
    Pick some $\theta_0\in\Omega_1\setminus\Omega_0.$
    Then clearly, we have the following:
    \emph{(i)} $J(\theta_0)=\frac{\partial F_L}{\partial\theta}\Big\vert_{\theta=\theta_0}$
    and \emph{(ii)} $\rank(J(\theta_0))=N.$ 
    This implies that there exists $\theta'\in\RR^p$ such that $\left(\frac{\partial F_L}{\partial\theta}\Big\vert_{\theta=\theta_0}\right)\theta'=Y$
    and thus,
    \begin{align*}
	y_i=
	\left(\left(\frac{\partial F_L}{\partial\theta}\Big\vert_{\theta=\theta_0}\right)\theta'\right)_{i}
	=\inner{\frac{\partial f_L(\theta,x_i)}{\partial\theta}\Big\vert_{\theta_0}, \theta'}
	=\lim_{\epsilon\to 0} \underbrace{\frac{f_L(\theta_0+\epsilon\theta',x_i)-f_L(\theta_0,x_i)}{\epsilon}}_{=:h_\epsilon(x_i)},\quad\forall\,i\in[N] .
    \end{align*}
    The result follows by noting that $h_\epsilon(x_i)$ can be implemented by a network of the same depth with twice more neurons at every hidden layer.

\subsection{Proof of Lemma \ref{lem:norm_W_SIGMA_wL}}\label{app:norm_W_SIGMA_wL}
    By a change of index $k+1\to k$, it is equivalent to prove the following:
    \begin{align*}
	\norm{\Sigma_{k}(x) \left(\prod_{l=k+1}^{L-1} W_l\Sigma_l(x)\right) W_L }_2^2
	= \bigTheta{\beta_L^2\, n_{k} \prod_{l=k+1}^{L-1} n_l\beta_l^2}.
    \end{align*}
    Let $B=\Sigma_{k}(x) \left(\prod_{l=k+1}^{L-1} W_l\Sigma_l(x)\right).$
    By Lemma \ref{lem:frobnorm_W_SIGMA}, $\norm{B}_F^2  = \bigTheta{n_k \prod_{l=k+1}^{L-1} n_l\beta_l^2}$
    w.p.\ at least $1-\sum_{l=1}^{L-1}\bigexp{-\bigOmg{n_l}}-\bigexp{-\bigOmg{d}}.$
    Moreover, one can also show that with a similar probability,
    \begin{align*}
	\norm{B}_{\op}^2=\bigO{\frac{n_k}{\min_{l\in[k,L-1]}n_l} \prod_{l=k+1}^{L-1} n_l\beta_l^2} .
    \end{align*}
    The proof of this is postponed below.
    Let us condition on the intersection of these two events of $(W_l)_{l=1}^{L-1}.$
    Then, by Hanson-Wright inequality (see \HWineq), we have
    \begin{align*}
	\frac{1}{2}\E_{W_L}\norm{BW_L}_2^2\leq \norm{BW_L}_2^2 \leq \frac{3}{2}\E_{W_L}\norm{BW_L}_2^2 .
    \end{align*}
    w.p.\ at least $1-e^{-\bigOmg{ \norm{B}_F^2/\norm{B}_{\op}^2 } }$ over $W_L.$ 
    Plugging the above bounds leads to the desired result. 
    
    In the remainder of this proof, we verify the above bound of $\norm{B}_{\op}^2$. 
    Concretely, we want to show that for every $p,q\in[L-1]$, 
    the following holds w.p.\ at least $1-\sum_{l=p-1}^{q}\bigexp{-\bigOmg{n_l}}$
    \begin{align}\label{eq:boundinduction}
	\norm{\prod_{l=p}^{q} W_l\Sigma_l(x)}_{\op}^2
	=\bigO{ \frac{\prod_{l=p-1}^q n_l}{\min_{l\in[p-1,q]}n_l}\, \prod_{l=p}^{q}\beta_l^2} .
    \end{align}
    Given that, the bound of $\norm{B}_{\op}^2$ follows immediately by letting $p=k+1, q=L-1$, and noting $\norm{\Sigma_k(x)}_{\op}\leq 1.$
    The proof of \eqref{eq:boundinduction} is by induction over the length $s=q-p$.
    First, \eqref{eq:boundinduction} holds for $s=0$ 
    since $\norm{W_p\Sigma_p(x)}_{\op}^2\leq \norm{W_p}_{\op}^2 = \bigO{\beta_p^2\max(n_p,n_{p-1})}$ where the last estimate
    follows from the standard bounds on the operator norm of Gaussian matrices (see \GaussOpNorm).
    Suppose that \eqref{eq:boundinduction} holds for $p, q$ such that $q-p\le s-1$, 
    and we want to prove it for all pairs $p,q$ with $q-p=s$. 
    It suffices to provide bound for one pair of $(p,q)$ and then do a union bound over all possible pairs.
    In the following, let
    \begin{align*}
	j=\argmin_{l\in[p-1,q]} n_l, \quad t=\argmin_{l\in[p-1,q]\setminus\Set{j}} n_l.
    \end{align*}
    We analyze three cases below. In the first case,
    namely $j\in[p,q-1]$, then
    \begin{align*}
	\norm{\prod_{l=p}^{q} W_l\Sigma_l(x)}_{\op}^2
	&\leq\norm{\prod_{l=p}^{j} W_l\Sigma_l(x)}_{\op}^2\, \norm{\prod_{l=j+1}^{q} W_l\Sigma_l(x)}_{\op}^2
	=\bigO{ \frac{\prod_{l=p-1}^{j} n_l}{\min_{l\in[p-1,j]}n_l}\, \frac{\prod_{l=j}^q n_l}{\min_{l\in[j,q]}n_l}\, \prod_{l=p}^{q}\beta_l^2}\\
	&=\bigO{ \frac{\prod_{l=p-1}^q n_l}{n_j}\, \prod_{l=p}^{q}\beta_l^2 }
	=\bigO{ \frac{\prod_{l=p-1}^q n_l}{\min_{l\in[p-1,q]}n_l}\, \prod_{l=p}^{q}\beta_l^2 },
    \end{align*}
    where the first equality follows from our induction assumption, the second equality follows from the current choice of $j.$
    In the second case, if $j=q$ and $t\in[p,q-1]$, then similarly one has
    \begin{align*}
	\norm{\prod_{l=p}^{q} W_l\Sigma_l(x)}_{\op}^2
	&\leq\norm{\prod_{l=p}^{t} W_l\Sigma_l(x)}_{\op}^2\, \norm{\prod_{l=t+1}^{q} W_l\Sigma_l(x)}_{\op}^2
	=\bigO{ \frac{\prod_{l=p-1}^{t} n_l}{\min_{l\in[p-1,t]}n_l}\, \frac{\prod_{l=t}^q n_l}{\min_{l\in[t,q]}n_l}\, \prod_{l=p}^{q}\beta_l^2 }\\
	&=\bigO{ \frac{\prod_{l=p-1}^{t} n_l}{n_t}\, \frac{\prod_{l=t}^q n_l}{n_q}\, \prod_{l=p}^{q}\beta_l^2 }
	=\bigO{ \frac{\prod_{l=p-1}^q n_l}{\min_{l\in[p-1,q]}n_l}\, \prod_{l=p}^{q}\beta_l^2 } .
    \end{align*}
    It remains to handle the case in which either $(j=p-1)$ or $(j=q \textrm{ and } t=p-1)$. 
    To do so, we use an $\epsilon$-net argument.  
    Since $\norm{\Sigma_q(x)}_{\op}\leq 1$, it holds that
    \begin{align}\label{eq:scase}
	\norm{\prod_{l=p}^{q} W_l\Sigma_l(x)}_{\op}^2 
	\leq\norm{\left(\prod_{l=p}^{q-1} W_l\Sigma_l(x)\right) W_q}_{\op}^2.
    \end{align}
    Furthermore, by using Lemma 4.4.1 of \citep{vershynin2018high},
    \begin{align}\label{eq:scase2}
	\norm{\left(\prod_{l=p}^{q-1} W_l\Sigma_l(x)\right) W_q}_{\op}^2
	\leq 4 \sup_{y\in{\sf N}_{1/2}^{p-1}} \norm{\underbrace{y^T \left(\prod_{l=p}^{q-1} W_l\Sigma_l(x)\right)}_{=:z^T} W_q}_2^2,
    \end{align}
    where ${\sf N}_{1/2}^{p-1}$ is a $\frac{1}{2}$-net of the unit sphere in $\RR^{n_{p-1}}.$ 
    Fix $y\in {\sf N}_{1/2}^{p-1}$, and let $z$ be defined as above, then clearly $z$ is independent of $W_q$, and it holds by induction assumption
    \begin{align}\label{eq:sz}
	\norm{z}_2^2=\bigO{ \frac{\prod_{l=p-1}^{q-1} n_l}{\min_{l\in[p-1,q-1]}n_l} \prod_{l=p}^{q-1}\beta_l^2 }
    \end{align}
    w.p.\ at least $1-\sum_{l=p}^{q-1} \bigexp{-\bigOmg{n_l}}$ over $(W_l)_{l=1}^{q-1}.$
    Conditioned on this event of the first $q-1$ layers, 
    let us study concentration bound for $\norm{z^TW_q}_2^2$ where the only randomness is over $W_q.$
    Note that $\norm{z^T W_q}_2^2=\sum_{j=1}^{n_q} \inner{z,(W_q)_{:j}}^2$ 
    and $\norm{\inner{z,(W_q)_{:j}}^2}_{\psi_1}\leq c_1\beta_q^2\norm{z}_2^2.$
    Thus by Bernstein's inequality (see \Bernstein),
    \begin{align*}
	\PP\left(\abs{\norm{z^T W_q}_2^2-\E_{W_q}\norm{z^T W_q}_2^2}>t\right)
	\leq\bigexp{-c_2\min\left(\frac{t}{c_1\beta_q^2\norm{z}_2^2}, \frac{t^2}{n_q c_1^2\beta_q^4\norm{z}_2^4}\right)},
    \end{align*}
    for some constant $c_2$. 
    By plugging $t=C c_1 \max(n_q, n_{p-1})\beta_q^2\norm{z}_2^2/c_2$ for some $C>\max(c_2,\log 5)$, 
    and $\E_{W_q}\norm{z^T W_q}_2^2=n_q\beta_q^2\norm{z}_2^2$, 
    one obtains $\norm{z^T W_q}_2^2=\bigO{ \max(n_q,n_{p-1})\beta_q^2\norm{z}_2^2}$
    w.p.\ at least $1-e^{-C\max(n_q,n_{p-1})}.$
    Taking the union bound over $y\in{\sf N}_{1/2}^{p-1}$, 
    we get 
    \begin{align*}
	\sup_{y\in{\sf N}_{1/2}^{p-1}} \norm{z^T W_q}_2^2
	=\sup_{y\in{\sf N}_{1/2}^{p-1}} \norm{y^T \left(\prod_{l=p}^{q-1} W_l\Sigma_l(x)\right) W_q}_2^2 
	=\bigO{ \max(n_q,n_{p-1})\beta_q^2\norm{z}^2}
    \end{align*}
    w.p.\ at least
	$1-\abs{{\sf N}_{1/2}^{p-1}} e^{-C\max(n_q,n_{p-1})} = 1-e^{-\bigOmg{\max(n_q,n_{p-1})}}$, 
    where we used the fact that $\abs{{\sf N}_{1/2}^{p-1}}\leq 5^{n_{p-1}}$ and $C>\log 5.$
    This combined with \eqref{eq:scase},\eqref{eq:scase2} and \eqref{eq:sz} implies
    \begin{align*}
	\norm{\prod_{l=p}^{q} W_l\Sigma_l(x)}_{\op}^2
	=\bigO{ \max(n_q,n_{p-1}) \beta_q^2 \, \frac{\prod_{l=p-1}^{q-1} n_l}{\min_{l\in[p-1,q-1]}n_l} \prod_{l=p}^{q-1}\beta_l^2 }
	=\bigO{ \frac{\prod_{l=p-1}^{q} n_l}{\min_{l\in[p-1,q]}n_l} \prod_{l=p}^{q}\beta_l^2 },
    \end{align*}
    where the last estimate follows from the current conditions on $(j,t).$
    To summarize, we have shown that \eqref{eq:boundinduction} holds for every given pair $(p,q)$ such that $q-p=s.$
    Taking the union bound over all these pairs finishes the proof. 
    Finally, note that doing the union bound above does not affect the probability of the final result
    since the number of all possible pairs is only a constant.

\section{Missing Proofs from Section \ref{sec:svmin_Fk}}
\subsection{Proof of Lemma \ref{lem:svmin_Fk_chernoff}}\label{app:lem:svmin_Fk_chernoff}
    For a subgaussian random variable $Z$, recall that $\PP(Z>t)\leq\exp(-c\, t^2 / \norm{Z}_{\psi_2}^2),$ 
    where $c$ is an absolute constant.
    In the following, let $t=\frac{4\beta_k\norm{F_{k-1}}_F}{c} \sqrt{\max\Big(1, \log\frac{8\beta_k^2\norm{F_{k-1}}_F^2}{c\,\lambda}\Big) }.$
    Let us denote the shorthand $W_k=[w_1,\ldots,w_{n_k}]\in\RR^{n_{k-1}\times n_k}$,
    and denote by $A\in\RR^{N\times n_k}$  a matrix such that
    $A_{:j}=\sigma(F_{k-1}w_j)\, \mathbbm{1}_{\norm{\sigma(F_{k-1}w_j)}_2\leq t}$ for all $j\in[n_k].$
    Let
    \begin{align*}
	&G = \E_{w\distas{}\mathcal{N} (0,\beta_k^2\,\Id_{n_{k-1}})} \left[ \sigma(F_{k-1}w)\sigma(F_{k-1}w)^T \right],\\
	&\hat{G} = \E_{w\distas{}\mathcal{N}(0,\beta_k^2\,\Id_{n_{k-1}})} \left[ \sigma(F_{k-1}w)\sigma(F_{k-1}w)^T \, \mathbbm{1}_{\norm{\sigma(F_{k-1}w)}_2\leq t} \right].
    \end{align*}
    Note $\lambda=\evmin{G}$, $\evmin{F_kF_k^T} \geq\evmin{AA^T}$ and
	$\evmax{A_{:j}A_{:j}^T} \leq t^2 .$
    By Matrix Chernoff inequality (see \MatrixChernoff), 
    it holds for every $\epsilon\in[0,1)$
    \begin{align*}
	\PP\Big( \evmin{AA^T}\leq (1-\epsilon)\evmin{\E AA^T} \Big) 
	\leq N\left[\frac{e^{-\epsilon}}{(1-\epsilon)^{1-\epsilon}}\right]^{\evmin{\E AA^T}/t^2}.
    \end{align*}
    Pick $\epsilon=1/2$. Then,
    \begin{align*}
	\PP\left( \evmin{AA^T} \leq n_k\evmin{\hat{G}}/2 \right)
	\leq \bigexp{-c_1\, n_k\evmin{\hat{G}}/t^2+\log N}.
    \end{align*}
    Thus, for $n_k\geq\frac{t^2}{c_1\evmin{\hat{G}}} \log\frac{N}{\delta}$
    we have $\evmin{AA^T}\geq \frac{n_k\evmin{\hat{G}}}{2}$ w.p.\ $\geq 1-\delta.$
    Moreover, 
    \begin{align*}
	\norm{\hat{G}-G}_2
	&\leq \E\norm{\sigma(F_{k-1}w)\sigma(F_{k-1}w)^T\, \mathbbm{1}_{\norm{\sigma(F_{k-1}w)}_2\leq t} - \sigma(F_{k-1}w)\sigma(F_{k-1}w)^T}_2 \\
	&= \E\left[ \norm{\sigma(F_{k-1}w)}_2^2 \, \mathbbm{1}_{\norm{\sigma(F_{k-1}w)}_2>t} \right]\\
	&= \int_{s=0}^{\infty} \PP\left( \norm{\sigma(F_{k-1}w)}_2 \, \mathbbm{1}_{\norm{\sigma(F_{k-1}w)}_2>t} > \sqrt{s} \right) ds\\
	&= \int_{s=0}^{\infty} \PP\left( \norm{\sigma(F_{k-1}w)}_2>t \right) \PP\left( \norm{\sigma(F_{k-1}w)}_2 > \sqrt{s} \right) ds\\
	&\leq \int_{s=0}^{\infty} \bigexp{-c\frac{t^2 + s}{4\beta_k^2\norm{F_{k-1}}_F^2}} ds\\
	&\leq \lambda/2 ,
    \end{align*}
    where the second inequality uses the fact that $\norm{\norm{\sigma(F_{k-1}w)}_2}_{\psi_2}\leq 2\beta_k\norm{F_{k-1}}_F.$
    It follows that $\evmin{\hat{G}}\geq\lambda/2.$
    In total, for $n_k\geq \frac{2 t^2}{c_1\lambda} \log\frac{N}{\delta}$, it holds w.p.\ at least $1-\delta$ that
    \begin{align*}
	\svmin{F_k}^2
	=\evmin{F_kF_k^T}
	\geq \evmin{AA^T}
	\geq n_k \evmin{\hat{G}}/2
	\geq n_k \lambda/4,
    \end{align*}
    where we used the condition $n_k\geq N$ in the above equality.

\subsection{Proof of Lemma \ref{lem:bound_lambda_star_k}}\label{app:lem:bound_lambda_star_k}
    Let $D=\diag(\norm{(F_{k})_{1:}}_2, \ldots, \norm{(F_{k})_{N:}}_2)$ and $\hat{F}_{k}= D^{-1} F_{k}.$
    Then, by the homogeneity of $\sigma$, we have
    \begin{align*}
	\evmin{ \E [\sigma(F_{k}w)\sigma(F_{k}w)^T] } 
	&=\evmin{ D \E\left[\sigma(\hat{F}_{k}w)\sigma(\hat{F}_{k}w)^T\right] D } \\
	&=\beta_{k+1}^2 \evmin{D \left[\mu_0(\sigma)^2 1_N 1_N^T+\sum_{s=1}^{\infty}\mu_s(\sigma)^2 (\hat{F}_{k}^{*s}) (\hat{F}_{k}^{*s})^T \right] D}\\
	&\geq\beta_{k+1}^2\, \mu_r(\sigma)^2 \evmin{D (\hat{F}_{k}^{*r}) (\hat{F}_{k}^{*r})^T D}\\
	&=\beta_{k+1}^2\, \mu_r(\sigma)^2 \evmin{D^{-(r-1)} (F_{k}^{*r}) (F_{k}^{*r})^T D^{-(r-1)}}\\
	&\geq\beta_{k+1}^2\, \mu_r(\sigma)^2 \frac{\evmin{(F_{k}^{*r})(F_{k}^{*r})^T}}{\max_{i\in [N]}\norm{(F_{k})_{i:}}_2^{2(r-1)}} ,
    \end{align*}
where the second equality uses the Hermite expansion of $\sigma$ (for the proof see Lemma D.3 of \citep{QuynhMarco2020}).
\subsection{Proof of Lemma \ref{lem:Fk_vs_Fk_tilde}}\label{app:lem:Fk_vs_Fk_tilde}
\revision{    We have that
    \begin{equation}\label{eq:Hprod}
	\begin{split}
	    (F_{k}^{*r}) (F_{k}^{*r})^T 
	    &= \left(F_{k} F_{k}^T\right)^{\circ r}.
	\end{split}
    \end{equation}    
After some manipulations, we obtain
    \begin{equation}\label{eq:manipnew}
    \begin{split}
	F_{k}F_{k}^T 
	&= \tilde F_{k} \tilde F_{k}^T + \norm{\mu}_2^2 1_N1_N^T + \Lambda 1_N 1_N^T + 1_N 1_N^T\Lambda \\
	&= \tilde F_{k} \tilde F_{k}^T + \left(\norm{\mu}_2 1_N+\frac{\Lambda 1_N}{\norm{\mu}_2}\right)\left(\norm{\mu}_2 1_N+\frac{\Lambda 1_N}{\norm{\mu}_2}\right)^T -\frac{\Lambda 1_N 1_N^T\Lambda}{\norm{\mu}_2^2} \\
	&\succeq \tilde F_{k} \tilde F_{k}^T -\frac{\Lambda 1_N 1_N^T\Lambda}{\norm{\mu}_2^2},
    \end{split} 
    \end{equation}   
    where in the last passage we use that $\left(\norm{\mu}_2 1_N+\frac{\Lambda 1_N}{\norm{\mu}_2}\right)\left(\norm{\mu}_2 1_N+\frac{\Lambda 1_N}{\norm{\mu}_2}\right)^T$ is PSD. Note that the RHS of \eqref{eq:manipnew} is also PSD since 
    \begin{equation*}
   \tilde F_{k} \tilde F_{k}^T \succeq \frac{\tilde F_{k} \mu\mu^T\tilde F_{k}^T}{\norm{\mu}_2^2}=\frac{\Lambda 1_N 1_N^T\Lambda}{\norm{\mu}_2^2}.
    \end{equation*}
    Hence, the $r$-th Hadamard power of the LHS of \eqref{eq:manipnew} is an upper bound (in the PSD sense) of the $r$-th Hadamard power of the RHS of \eqref{eq:manipnew}, which concludes the proof.}

\subsection{Proof of Lemma \ref{lem:svmin_Fk_tilde_Khatri_Rao}}\label{app:lem:svmin_Fk_tilde_Khatri_Rao}
\revision{Note that 
\begin{equation}
\left(\tilde F_{k} \tilde F_{k}^T -\frac{\Lambda 1_N 1_N^T\Lambda}{\norm{\mu}_2^2}\right)^{\circ r} = \left(\tilde F_{k} \tilde F_{k}^T\right)^{\circ r} + \sum_{i=1}^r {r \choose i}(-1)^i\frac{\Lambda^i\left(\tilde F_{k} \tilde F_{k}^T\right)^{\circ (r-i)}\Lambda^i}{\norm{\mu}_2^{2i}}.
\end{equation}
Thus, an application of Weyl's inequality gives
\begin{equation}\label{eq:Win}
\evmin{\left(\tilde F_{k} \tilde F_{k}^T -\frac{\Lambda 1_N 1_N^T\Lambda}{\norm{\mu}_2^2}\right)^{\circ r}}\ge \evmin{\left(\tilde F_{k} \tilde F_{k}^T\right)^{\circ r}}-\left\|\sum_{i=1}^r {r \choose i}(-1)^i\frac{\Lambda^i\left(\tilde F_{k} \tilde F_{k}^T\right)^{\circ (r-i)}\Lambda^i}{\norm{\mu}_2^{2i}}\right\|_{\rm op}.
\end{equation}}

\revision{We start by bounding the term $\evmin{\left(\tilde F_{k} \tilde F_{k}^T\right)^{\circ r}}$.} From Gershgorin circle theorem, one obtains
    \begin{align}
	&\evmin{(\tilde{F}_{k}^{*r}) (\tilde{F}_{k}^{*r})^T}  \geq \min\limits_{i\in[N]} \|(\tilde{F}_k)_{i:}\|_2^{2r} - N \max\limits_{i\neq j} | \langle(\tilde{F}_{k})_{i:}, (\tilde{F}_{k})_{j:}\rangle |^r , \label{eq:svmin_Fk_lob} \\
	&\evmin{(\tilde{F}_{k}^{*r}) (\tilde{F}_{k}^{*r})^T} \leq \max\limits_{i\in[N]} \|(\tilde{F}_k)_{i:}\|_2^{2r} + N \max\limits_{i\neq j} | \langle(\tilde{F}_{k})_{i:}, (\tilde{F}_{k})_{j:}\rangle |^r . \label{eq:svmin_Fk_upb}
    \end{align}
    By Lemma \ref{lem:square_fx}, it holds w.p.\ at least $1-N\bigexp{ -\bigOmg{ \frac{\min_{l\in[0,k]} n_l}{\prod_{l=1}^{k-1}\log(n_l)} } } - \sum_{l=1}^{k}\exp(-\bigOmg{n_l})$ that
    \begin{align}\label{eq:svmin_diag}
	\|(\tilde{F}_k)_{i:}\|_2^{2r}
	= \bigTheta{ \left(d \prod_{l=1}^{k} n_l\beta_l^2\right)^{r} }, \quad\forall\,i\in[N].
    \end{align}
    In the following, we bound the second term on the RHS of \eqref{eq:svmin_Fk_upb}.
    For a fixed $j\in[N]$, Lemma \ref{lem:square_fx} implies that w.p.\ at least $1-\bigexp{ -\bigOmg{ \frac{\min_{l\in[0,k]} n_l}{\prod_{l=1}^{k-1}\log(n_l)} } } - \sum_{l=1}^{k}\exp(-\bigOmg{n_l})$
    over $(W_l)_{l=1}^k$ and $x_j$, we have
    \begin{align}\label{eq:event_Fkj}
	\norm{(\tilde{F}_k)_{j:}}_2^2 
	= \bigTheta{ d \prod_{l=1}^{k} n_l\beta_l^2 } .
    \end{align}
    Moreover, Theorem \ref{thm:lip_const_fk} implies that w.p.\ at least $1-\sum_{l=1}^{k}\bigexp{-\bigOmg{n_l}}$ over $(W_l)_{l=1}^{k},$
    \begin{align}\label{eq:event_lip_fk}
	\norm{f_k(x)-\E_x f_k(x)}_{\Lip}^2
	=\bigO{ \frac{\prod_{l=0}^k n_l}{\min_{l\in[0,k]}n_l}\, \prod_{l=1}^{k-1}\log(n_l)\, \prod_{l=1}^{k}\beta_l^2 } .
    \end{align}
    Let us condition on the intersection of these two events of $(W_l)_{l=1}^k$ and $x_j$, and derive probability bounds over $x_i$, for every $i\neq j.$
    Let $h(x_i)=\inner{(\tilde{F}_{k})_{i:}, (\tilde{F}_{k})_{j:}}$ be a function of $x_i$, then
    \begin{align*}
	\norm{h}_{\Lip}^2 
	\leq \norm{(\tilde{F}_{k})_{j:}}_2^2 \norm{f_k(x_i)-\E_x f_k(x_i)}_{\Lip}^2 
	=\bigO{\left(d\prod_{l=1}^{k}n_l\beta_l^2\right)^2  \frac{\prod_{l=1}^{k-1}\log(n_l)}{\min_{l\in[0,k]}n_l} } ,
    \end{align*}
    where the last estimate follows from \eqref{eq:event_Fkj} and \eqref{eq:event_lip_fk}.
    Using Assumption \ref{ass:data_dist2}, followed by a union bound over $\Set{x_i}_{i\neq j}$, we have for every $t>0$ that
    \begin{align}\label{eq:TSineq3}
	\mathbb P\left( \max_{i\in[N], i\neq j} \abs{\inner{(\tilde{F}_{k})_{i:}, (\tilde{F}_{k})_{j:}}} \geq t \right)
	&\leq (N-1) \bigexp{- \frac{t^2}{ \bigO{\left(d\prod_{l=1}^{k}n_l\beta_l^2\right)^2  \frac{\prod_{l=1}^{k-1}\log(n_l)}{\min_{l\in[0,k]}n_l} } } } .
    \end{align}
    Pick $t=N^{-1/(r-0.1)} \left( d \prod_{l=1}^{k} n_l\beta_l^2 \right)$. 
    Then, taking the intersection bound with \eqref{eq:event_Fkj} and \eqref{eq:event_lip_fk} yields 
    \begin{align}\label{eq:svmin_offdiag0}
	N\max\limits_{i\in[N], i\neq j} | \langle(\tilde{F}_{k})_{i:}, (\tilde{F}_{k})_{j:}\rangle |^r 
	\leq N\frac{\left( d \prod_{l=1}^{k} n_l\beta_l^2 \right)^r}{N^{r/(r-0.1)}} 
	= \littleO{ \left(d \prod_{l=1}^{k} n_l\beta_l^2\right)^r } 
    \end{align}
    w.p.\ at least
    \begin{align*}
	1 - (N-1) \bigexp{ - \bigOmg{ \frac{\min_{l\in[0,k]}n_l}{N^{2/(r-0.1)} \prod_{l=1}^{k-1}\log(n_l)} } } - \sum_{l=1}^{k} \bigexp{-\bigOmg{n_l}} .
    \end{align*}
    Since this holds for every given $x_j$, taking the union bound over $j\in[N]$ yields that 
    \begin{align}\label{eq:svmin_offdiag}
	N\max\limits_{i\neq j} | \langle(\tilde{F}_{k})_{i:}, (\tilde{F}_{k})_{j:}\rangle |^r
	= \littleO{ \left(d \prod_{l=1}^{k} n_l\beta_l^2\right)^r } 
    \end{align}
    w.p.\ at least
    \begin{align}\label{eq:problemma2}
	1 - N^2 \bigexp{ - \bigOmg{ \frac{\min_{l\in[0,k]}n_l}{N^{2/(r-0.1)} \prod_{l=1}^{k-1}\log(n_l)} } } - N \sum_{l=1}^{k} \bigexp{-\bigOmg{n_l}} .
    \end{align}
\revision{    Combining \eqref{eq:svmin_Fk_lob}, \eqref{eq:svmin_Fk_upb}, \eqref{eq:svmin_diag}, \eqref{eq:svmin_offdiag} gives that, with probability lower bounded by \eqref{eq:problemma2},
    \begin{equation}\label{eq:evmincf}
   \evmin{\left(\tilde F_{k} \tilde F_{k}^T\right)^{\circ r}} = \bigTheta{ \left( d \prod_{l=1}^{k} n_l\beta_l^2 \right)^r }. 
    \end{equation}
By applying again Gershgorin circle theorem and following similar passages, we also have that 
\begin{equation}\label{eq:bdop0}
\max_{i\in\{1, \ldots, r/2\}}   \left\|\left(\tilde F_{k} \tilde F_{k}^T\right)^{\circ (r-i)}\right\|_{\rm op} = \bigO{ \left( d \prod_{l=1}^{k} n_l\beta_l^2 \right)^{r-i} }
\end{equation}    
    w.p.\ at least 
    \begin{equation}\label{eq:problemma3}
	1 - N^2 \bigexp{ - \bigOmg{ \frac{\min_{l\in[0,k]}n_l}{N^{2/(r/2-0.1)} \prod_{l=1}^{k-1}\log(n_l)} } } - N \sum_{l=1}^{k} \bigexp{-\bigOmg{n_l}} .    
    \end{equation}
    Furthermore, by using \eqref{eq:event_Fkj} and that the Frobenius norm upper bounds the operator norm, we also obtain the following simple bound
    \begin{equation}\label{eq:bdop1}
       \left\|\left(\tilde F_{k} \tilde F_{k}^T\right)^{\circ (r-i)}\right\|_{\rm op}\le \left\|\left(\tilde F_{k} \tilde F_{k}^T\right)^{\circ (r-i)}\right\|_{F}\le  \bigO{ N\left(d \prod_{l=1}^{k} n_l\beta_l^2\right)^{r-i} } 
    \end{equation}
holding w.p.\ at least    $1-N\bigexp{ -\bigOmg{ \frac{\min_{l\in[0,k]} n_l}{\prod_{l=1}^{k-1}\log(n_l)} } } - \sum_{l=1}^{k}\exp(-\bigOmg{n_l})$.
    }
    
\revision{    Next, we upper bound the operator norm in the RHS of \eqref{eq:Win}. As the operator norm is sub-additive and sub-multiplicative, we have that
    \begin{equation}\label{eq:subsplit}
    \begin{split}
    &\left\|\sum_{i=1}^r {r \choose i}(-1)^i\frac{\Lambda^i\left(\tilde F_{k} \tilde F_{k}^T\right)^{\circ (r-i)}\Lambda^i}{\norm{\mu}_2^{2i}}\right\|_{\rm op}\le \sum_{i=1}^r {r \choose i}\left\|\frac{\Lambda}{\norm{\mu}_2}\right\|_{\rm op}^{2i}\left\|\left(\tilde F_{k} \tilde F_{k}^T\right)^{\circ (r-i)}\right\|_{\rm op}\\
    &\hspace{1em}= \sum_{i=1}^{r/2} {r \choose i}\left\|\frac{\Lambda}{\norm{\mu}_2}\right\|_{\rm op}^{2i}\left\|\left(\tilde F_{k} \tilde F_{k}^T\right)^{\circ (r-i)}\right\|_{\rm op}+\sum_{i=r/2+1}^r {r \choose i}\left\|\frac{\Lambda}{\norm{\mu}_2}\right\|_{\rm op}^{2i}\left\|\left(\tilde F_{k} \tilde F_{k}^T\right)^{\circ (r-i)}\right\|_{\rm op}:=S_1+S_2.
\end{split}   
    \end{equation}
 }
 
 \revision{
Let $h:\RR^d\to\RR$ be a function over a random sample $x$, defined as $h(x)=\inner{f_k(x),\mu}.$
    Then, $\Lambda_{ii} = h(x_i)-\mathbb E_x[h(x)].$
    Since $\norm{h}_{\Lip}^2\leq\norm{\mu}_2^2\norm{f_k}_{\Lip}^2$, it holds
    \begin{equation}\label{eq:TSineq}
	\PP\left(\abs{\Lambda_{ii}}\geq t\right)
	\leq \exp\left(-\frac{t^2}{2\norm{\mu}_2^2\norm{f_k}_{\Lip}^2}\right).
    \end{equation}
    By Lemma \ref{lem:norm_of_expected_feature_map}, it holds w.p.\ at least $1 - \sum_{l=1}^k\bigexp{-\bigOmg{n_l}}$ over $(W_l)_{l=1}^k$ that
    \begin{align}\label{eq:event_Efk}
	\norm{\mu}_2^2 = \bigTheta{ d\prod_{l=1}^{k} n_l \beta_l^2 } .
    \end{align}    
   Also, Theorem \ref{thm:lip_const_fk} shows that w.p.\ at least $1-\sum_{l=1}^{k}\bigexp{-\bigOmg{n_l}}$ over $(W_l)_{l=1}^{k}$, 
    \begin{align}\label{eq:event_lip_fk2}
	\norm{f_k}_{\Lip}^2
	=\bigO{ \frac{\prod_{l=0}^k n_l}{\min_{l\in[0,k]}n_l}\, \prod_{l=1}^{k-1}\log(n_l)\, \prod_{l=1}^{k}\beta_l^2} .
    \end{align}
   Now, pick $t=\|\mu\|_2^2N^{-1/(r/2-0.1)}$ in \eqref{eq:TSineq}. Then, taking the union bound over all the samples and over the events in \eqref{eq:event_Efk} and \eqref{eq:event_lip_fk2}, we conclude that
       \begin{align}\label{eq:Ln}
	\left\|\frac{\Lambda}{\norm{\mu}_2}\right\|_{\rm op}^2=\bigO{N^{-2/(r/2-0.1)}d\prod_{l=1}^{k} n_l \beta_l^2}
    \end{align}
    w.p.\ at least 
    \begin{align*}
	1-N \bigexp{ - \bigOmg{ \frac{\min_{l\in[0,k]}n_l}{N^{2/(r/2-0.1)} \prod_{l=1}^{k-1}\log(n_l)} } }  -\sum_{l=1}^{k}\bigexp{-\bigOmg{n_l}}.
    \end{align*} 
 By combining \eqref{eq:bdop0} and \eqref{eq:Ln}, we have that
 \begin{equation}\label{eq:subsplit1}
 S_1\le \bigO{N^{-2/(r/2-0.1)}\left(d\prod_{l=1}^{k} n_l \beta_l^2\right)^r}=\littleO{ \left(d \prod_{l=1}^{k} n_l\beta_l^2\right)^r } 
 \end{equation}
        w.p.\ lower bounded by \eqref{eq:problemma3}. Furthermore, by combining \eqref{eq:bdop1} and \eqref{eq:Ln}, we have that
 \begin{equation}\label{eq:subsplit2}
 S_2\le \bigO{N^{1-r/(r/2-0.1)}\left(d\prod_{l=1}^{k} n_l \beta_l^2\right)^r}=\littleO{ \left(d \prod_{l=1}^{k} n_l\beta_l^2\right)^r } 
  \end{equation}
w.p.\ lower bounded by \eqref{eq:problemma3}. By combining \eqref{eq:Win}, \eqref{eq:evmincf}, \eqref{eq:subsplit}, \eqref{eq:subsplit1} and \eqref{eq:subsplit2}, the desired result \eqref{eq:dresnew} follows.
}

\section{Missing Proofs from Section \ref{sec:lip_const_fk}}
\begin{definition}\label{def:PWL}
    A subset $A\subseteq\RR^n$ is called a polyhedron if it is the intersection of a finite family of (closed) half-spaces.
    A function $f:\RR^n\to\RR^m$ is called piecewise linear 
    if there exist a finite family of polyhedra $\Set{P_i}_{i=1}^{r}$ such that
    $\RR^n=\cup_{i=1}^{r} P_i$ and $f$ coincides with a linear function on each $P_i$.
\end{definition}
The following lemma establishes a formal connection between ReLU networks and PWL functions. Its proof is contained in Appendix \ref{app:lem:NN_PWL}.
\begin{lemma}\label{lem:NN_PWL}
    For every $k\in[L]$, $f_k,g_k:\RR^d\to\RR^{n_k}$ as defined in \eqref{eq:def_feature_map} are piecewise linear functions.
\end{lemma}
An equivalent way of defining piecewise linear maps is the following, see e.g. \citep{Gorokhovik2011}.
\begin{lemma}\label{lem:PWL_polyhedra}
    A function $f:\RR^n\to\RR^m$ is piecewise linear 
    if and only if there exist a finite family of polyhedra $\Set{P_i}_{i=1}^{T}$ 
    and matrices $\Set{A_i}_{i=1}^{T}\in\RR^{m\times n}$ such that:
    \begin{enumerate}
	\item $\RR^n=\bigcup_{i=1}^{T} P_i,$
	\item $\interior(P_i)\neq\emptyset,\quad\forall\,i\in[T],$
	\item $\interior(P_i)\cap \interior(P_j)=\emptyset\quad\forall\,i\neq j,$
	\item $f(x)=A_ix$ for every $x\in P_i.$
    \end{enumerate}
\end{lemma}

\subsection{Proof of Theorem \ref{thm:lip_const_fk}}\label{app:thm:lip_const_fk}
    Let $h_{p\to q}:\RR^{n_p}\to\RR^{n_q}$ be defined as
    \begin{align*}
	h_{p\to q} = A_q\circ\hat{\sigma}_{q-1}\circ A_{q-1}\circ \ldots\circ\hat{\sigma}_{p+1}\circ A_{p+1} ,
    \end{align*}
    where the mapping $A_l:\RR^{n_{l-1}}\to\RR^{n_l}$ is given by $A_l(x)=W_l^T x$,
    and the mapping $\hat{\sigma}_l:\RR^{n_l}\to\RR^{n_l}$ is given by $\hat{\sigma}(x)=[\sigma(x_1),\ldots,\sigma(x_{n_l})]^T$ for every $x\in\RR^{n_l}.$
    By definition, it holds $g_k(x)=h_{0\to k}(x).$
    In the following, we prove that for every $0\leq p<q\leq L$, 
    it holds w.p.\ $\geq 1-\sum_{l=p-1}^{q}\bigexp{-\bigOmg{n_l}}$ that
    \begin{align}\label{eq:bound_hpq}
	\norm{h_{p\to q}}_{\Lip}
	=\bigO{ \frac{\prod_{l=p}^{q} n_l}{\min_{l\in[p,q]}n_l}\, \prod_{l=p+1}^{q-1}\log(n_l)\, \prod_{l=p+1}^{q}\beta_l^2} .
    \end{align}
    The desired result follows by letting $p=0, q=k$.
    The proof of \eqref{eq:bound_hpq} is by induction over the length $s=q-p$.
    First, \eqref{eq:bound_hpq} holds for $s=1.$
    Suppose that \eqref{eq:bound_hpq} holds for all $(p, q)$ such that $q-p\le s-1$, 
    and we want to prove it for all $(p,q)$ with $q-p=s$. 
    It suffices to show the result for one pair and then do a union bound over all the possible pairs.
    Let us define 
    \begin{align*}
	j=\argmin_{l\in[p,q]} n_l,\quad t=\argmin_{l\in[p,q]\setminus\Set{j}} n_l.
    \end{align*}
    Consider three cases below. 
    In the first case, $j\in[p+1,q-1]$. By noting that
    \begin{align*}
	h_{p\to q}=h_{j\to q} \circ \hat{\sigma}_j \circ h_{p\to j}
    \end{align*}
    and using the Lipschitz property of a composition of Lipschitz continuous functions, one obtains
    \begin{align*}
	\norm{h_{p\to q}}_{\Lip}
	&\leq \norm{h_{p\to j}}_{\Lip} \norm{\hat{\sigma}_{j}}_{\Lip} \norm{h_{j\to q}}_{\Lip} \\
	&=\bigO{ \frac{\prod_{l=p}^{j} n_l}{\min_{l\in[p,j]}n_l}\, \prod_{l=p+1}^{j-1}\log(n_l)\, \frac{\prod_{l=j}^q n_l}{\min_{l\in[j,q]}n_l}\, \prod_{l=j+1}^{q-1}\log(n_l)\, \prod_{l=p+1}^{q}\beta_l^2}\\
	&=\bigO{ \frac{\prod_{l=p}^q n_l}{\min_{l\in[p,q]}n_l}\, \prod_{l=p+1}^{q-1}\log(n_l)\, \prod_{l=p+1}^{q}\beta_l^2 } ,
    \end{align*}
    where the first equality follows from induction assumption and $\norm{\hat{\sigma}}_{\Lip}\leq 1$, 
    the second equality follows from definition of $j.$
    In the second case, $j=q$ and $t\in[p+1,q-1]$, then similarly,
    \begin{align*}
	\norm{h_{p\to q}}_{\Lip}
	&\leq \norm{h_{p\to t}}_{\Lip} \norm{\hat{\sigma}_{t}}_{\Lip} \norm{h_{t\to q}}_{\Lip} \\
	&=\bigO{ \frac{\prod_{l=p}^{t} n_l}{\min_{l\in[p,t]}n_l}\, \prod_{l=p+1}^{t-1}\log(n_l)\, \frac{\prod_{l=t}^q n_l}{\min_{l\in[t,q]}n_l}\, \prod_{l=t+1}^{q-1}\log(n_l)\, \prod_{l=p+1}^{q}\beta_l^2}\\
	&=\bigO{ \frac{n_t \prod_{l=p}^q n_l}{n_t n_q}\, \prod_{l=p+1}^{q-1}\log(n_l)\, \prod_{l=p+1}^{q}\beta_l^2 } \\
	&=\bigO{ \frac{\prod_{l=p}^q n_l}{\min_{l\in[p,q]}n_l}\, \prod_{l=p+1}^{q-1}\log(n_l)\, \prod_{l=p+1}^{q}\beta_l^2 }. 
    \end{align*}
    It remains to handle the case where either $(j=p)$ or $(j=q \textrm{ and } t=p)$. 
    By Lemma \ref{lem:lip_const_fk}, it holds w.p.\ 1 over $(W_l)_{l=p+1}^{q-1}$ that
    there exists a set of $R$ tuples of diagonal matrices, say $\mathcal{D}=\Set{(\Sigma_{p+1}^1,\ldots,\Sigma_{q-1}^1),\ldots,(\Sigma_{p+1}^R,\ldots,\Sigma_{q-1}^R)}$, 
    with 0-1 entries on the diagonals such that 
    \begin{align}\label{eq:tuples_hpq}
	\norm{h_{p\to q}}_{\Lip}
	\leq\; \max\limits_{(\Sigma_{p+1},\ldots,\Sigma_{q-1})\in\mathcal{D}}\; 
	\norm{ \left(\prod_{l=p+1}^{q-1} W_l\Sigma_l\right) W_q }_{\op} .
    \end{align}
    According to Lemma \ref{lem:lip_const_fk}, $R$ can be interpreted as the maximum number of activation patterns of a $q-p$ layer network with layer widths $(n_p,n_{p+1},\ldots,n_q)$,
    where every hidden neuron has a definite sign pattern $\Set{-1,+1}.$
    Let $n_{\mathrm{max}}=\max_{l\in[p+1,q-1]}n_l$, then $R=\bigO{(n_{\mathrm{max}})^{n_p}}$ (see e.g. \citep{HaninRolnick2019, SerraEtal2018}). 
    Using the definition of operator norm and an $\epsilon$-net argument, the inequality \eqref{eq:tuples_hpq} becomes
    \begin{align}\label{eq:scase_hpq}
	\norm{h_{p\to q}}_{\Lip}
	&\leq\; \max\limits_{(\Sigma_{p+1},\ldots,\Sigma_{q-1})\in\mathcal{D}}\;\; \sup_{\norm{y}_2=1}\, \norm{y^T \left(\prod_{l=p+1}^{q-1} W_l\Sigma_l\right) W_q }_{2} \nonumber\\
	&\leq\; \max\limits_{(\Sigma_{p+1},\ldots,\Sigma_{q-1})\in\mathcal{D}}\;\; 2\sup_{y\in{\sf N}_{1/2}^{p}} \norm{\underbrace{y^T \left(\prod_{l=p+1}^{q-1} W_l\Sigma_l\right)}_{=:z^T} W_q}_2^2,
    \end{align}
    where ${\sf N}_{1/2}^{p}$ is a $\frac{1}{2}$-net of the unit sphere in $\RR^{n_p}$ and the last inequality follows from Lemma 4.4.1 in \citep{vershynin2018high}. 
    Fix $y\in {\sf N}_{1/2}^{p}$, and let $z$ be defined as above. Note that $z$ is independent of $W_q$.
    From the proof of Lemma \ref{lem:norm_W_SIGMA_wL}, we have
    \begin{align}\label{eq:sz_hpq}
	\norm{z}_2^2
	\leq\norm{\prod_{l=p+1}^{q-1} W_l\Sigma_l}_{\op}^2
	=\bigO{ \frac{\prod_{l=p}^{q-1} n_l}{\min_{l\in[p,q-1]}n_l} \prod_{l=p+1}^{q-1}\beta_l^2 }
    \end{align}
    w.p.\ at least $1-\sum_{l=p}^{q-1} \bigexp{-\bigOmg{n_l}}$ over $(W_l)_{l=p+1}^{q-1}.$
    Conditioned on the intersection of this event with the event \eqref{eq:tuples_hpq} of $(W_l)_{l=p+1}^{q-1}$,
    let us now study a concentration bound for $\norm{z^TW_q}_2^2$ where the only randomness is $W_q.$
    We have $\norm{z^T W_q}_2^2=\sum_{j=1}^{n_q} \inner{z,(W_q)_{:j}}^2$ 
    and $\norm{\inner{z,(W_q)_{:j}}^2}_{\psi_1}\leq c_1\beta_q^2\norm{z}_2^2.$
    Thus by Bernstein's inequality (see \Bernstein), 
    \begin{align*}
	\PP\left(\abs{\norm{z^T W_q}_2^2-\E_{W_q}\norm{z^T W_q}_2^2}>t\right)
	\leq\bigexp{-c_2\min\left(\frac{t}{c_1\beta_q^2\norm{z}_2^2}, \frac{t^2}{n_q c_1^2\beta_q^4\norm{z}_2^4}\right)},
    \end{align*}
    for some constant $c_2$. 
    Let $C=\max(c_2, 2).$ Then by substituting to the above inequality the values
    \begin{align*}
	t=\frac{Cc_1}{c_2}\max(n_q, n_p)\frac{\log(R)}{n_p}\beta_q^2\norm{z}_2^2,\quad 
	\E_{W_q}\norm{z^T W_q}_2^2=n_q\beta_q^2\norm{z}^2_2 ,
    \end{align*}
    we have w.p.\ at least $1-e^{-C\max(n_q,n_p)\log(R)/n_p}$ that
    \begin{align*}
	\norm{z^T W_q}_2^2=\bigO{ \max(n_q,n_p)\frac{\log(R)}{n_p} \beta_q^2\norm{z}^2_2 }.
    \end{align*}
    Now taking the union bound over $y\in{\sf N}_{1/2}^p$ and all tuples from $\mathcal{D}$,
    the RHS of \eqref{eq:scase_hpq} is bounded as
    \begin{align*}
	\max\limits_{(\Sigma_{p+1},\ldots,\Sigma_{q-1})\in\mathcal{D}}\;\; 2\sup_{y\in{\sf N}_{1/2}^{p}} \norm{z^T W_q}_2^2
	&=\bigO{ \max(n_q,n_p)\frac{\log(R)}{n_p}\beta_q^2\norm{z}^2_2}\\
	&=\bigO{ \max(n_q,n_p) \log(n_{\mathrm{max}}) \beta_q^2\norm{z}^2_2}
    \end{align*}
    w.p.\ at least 
    \begin{align*}
	1-R\abs{{\sf N}_{1/2}^p} e^{-C\max(n_q,n_p)\frac{\log(R)}{n_p}} \geq 1-e^{-\bigOmg{\max(n_q,n_p)}},
    \end{align*}
    where we used $\abs{{\sf N}_{1/2}^p}\leq 5^{n_p}$, $R=\bigO{(n_{\mathrm{max}})^{n_p}}$ and $C>1.$
    This combined with \eqref{eq:scase_hpq}, \eqref{eq:sz_hpq} implies
    \begin{align*}
	\norm{h_{p\to q}}_{\Lip}
	&=\bigO{ \max(n_q,n_p) \log(n_{\mathrm{max}}) \beta_q^2 \, \frac{\prod_{l=p}^{q-1} n_l}{\min_{l\in[p,q-1]}n_l} \prod_{l=p+1}^{q-1}\beta_l^2 }\\
	&=\bigO{ \frac{\prod_{l=p}^{q} n_l}{\min_{l\in[p,q]}n_l} \log(\max_{l\in[p+1,q-1]}n_l) \prod_{l=p+1}^{q}\beta_l^2 } \\
	&=\bigO{ \frac{\prod_{l=p}^{q} n_l}{\min_{l\in[p,q]}n_l} \prod_{l=p+1}^{q-1}\log(n_l) \prod_{l=p+1}^{q}\beta_l^2 } ,
    \end{align*}
    where the second estimate follows from the current value of $(j,t).$
    So, we have shown that \eqref{eq:bound_hpq} holds for every pair $(p,q)$ with $q-p=s.$
    Taking the union bound over all these pairs finishes the proof. 
    Note that this last step does not affect the final probability as the number of pairs is only a constant.

\subsection{Proof of Lemma \ref{lem:lip_const_fk}}\label{app:lem:lip_const_fk}
    Let $\gamma_d$ be the Lebesgue measure in $\RR^d.$
    Let us associate to $g_k:\RR^d\to\RR^{n_k}$ a set of polyhedra $\Set{P_i}_{i=1}^{T}$
    and matrices $\Set{A_i}_{i=1}^{T}\in\RR^{n_k\times n_d}$ as in Lemma \ref{lem:PWL_polyhedra}.
    First, let us show that
    \begin{align}\label{eq:Lip_op}
	\norm{g_k}_{\Lip} 
	=\max_{i\in[T]} \norm{A_i}_{\op}.
    \end{align}
    Pick any $x,y\in\RR^d.$ 
    By intersecting the line segment $[x,y]$ with the polyhedra,
    there exists a finite set of points $\Set{u_i}_{i=1}^{r}$ on $[x,y]$ such that:
    \emph{(i)} $u_0=x,u_r=y$, \emph{(ii)} $\norm{x-y}_2=\sum_{i=0}^{r-1}\norm{u_i-u_{i+1}}_2$, and \emph{(iii)} $[u_i,u_{i+1}]$ is contained in $P_{j_i}$ for some $j_i\in[T].$
    This implies
    \begin{align*}
	\norm{g_k(x)-g_k(y)}_2
	\leq \sum_{i=0}^{r-1} \norm{g_k(u_i)-g_k(u_{i+1})}_2
	= \sum_{i=0}^{r-1} \norm{A_{j_i} (u_i-u_{i+1})}_2
	&\leq \sum_{i=0}^{r-1} \norm{A_{j_i}}_{\op} \norm{u_i-u_{i+1}}_2\\
	&\leq \max_{i\in[T]} \norm{A_i}_{\op} \norm{x-y}_2 ,
    \end{align*}
    which means
    \begin{align*}
	\norm{g_k}_{\Lip} 
	= \sup_{x, y} \frac{\norm{g_k(x)-g_k(y)}_2}{\norm{x-y}_2}
	\leq \max_{i\in[T]} \norm{A}_{\op} .
    \end{align*}
    To show that the above inequality can be attained,
    let $i_*=\argmax_{i\in[T]}\norm{A_i}_{\op}.$ 
    Since $\interior(P_{i_*})\neq\emptyset,$ it holds $$\Setbar{\frac{x-y}{\norm{x-y}_2}}{x,y\in P_{i_*}}=\mathcal{S}^{n-1},$$
    where $\mathcal{S}^{n-1}$ denotes the unit sphere in $\RR^n$, and thus
    \begin{align*}
	\sup_{x, y} \frac{\norm{g_k(x)-g_k(y)}_2}{\norm{x-y}_2}
	\geq
	\sup_{x,y\in P_{i_*}} \frac{\norm{g_k(x)-g_k(y)}_2}{\norm{x-y}_2}
	=\sup_{x,y\in P_{i_*}} \frac{\norm{A_{i_*}(x-y)}_2}{\norm{x-y}_2}
	=\norm{A_{i_*}}_{\op} .
    \end{align*}
    This proves the equation \eqref{eq:Lip_op}.
    Next, let us define the following sets:
    \begin{align*}
	&S=\Setbar{x\in\RR^d}{f_{k-1}(x)=0},\\
	&B=\Setbar{x\in\RR^d\setminus S}{\exists\,l\in[k-1], i_l\in[n_l]: g_{l,i_l}(x)=0},\\
	&G=\RR^d\setminus(B\cup S).
    \end{align*}
    Let $\partial S=S\setminus\interior(S).$ 
    Then clearly, $\RR^d=G\cup B\cup\partial S\cup\interior(S).$
    Let us show that $\gamma_d(B)=\gamma_d(\partial S)=0.$
    By Lemma \ref{lem:NN_PWL}, $f_{k-1}$ is a PWL function,
    thus every level set of $f_{k-1}$ can be written as a union of finitely many polyhedra in $\RR^d.$
    This means that $\partial S$ is a union of finitely many polyhedra with dimension at most $d-1$, thus $\gamma_d(\partial S)=0.$
    Concerning the set $B$, note that for every $l\in[k-1], i_l\in[n_l]$,
    \begin{align*}
	g_{l,i_l}(x)
	=\sum_{i_0=1}^{d} \sum_{i_1=1}^{n_1}\ldots\sum_{i_{l-1}=1}^{n_{l-1}} \prod_{p=1}^{l} x_{i_0} (W_p)_{i_{p-1},i_p} \prod_{q=1}^{l-1} \mathbbm{1}_{g_{q,i_q}(x)>0} .
    \end{align*}
    By definition, any $x\in B$ satisfies $f_l(x)\neq 0$ for all $l\in[k-1]$.
    This implies that at each layer $q\in[k-1]$, there exists at least one active neuron, i.e. some $i_q\in[n_q]$ such that $g_{q,i_q}(x)>0$.
    Let $\mathcal{I}_l$ denote the set of active neurons that an input $x\in B$ may have at layer $l\in[k-1].$
    Then it holds
    \begin{align*}
	B\subseteq
	\bigcup_{l\in[k-1]}\;
	\bigcup_{i_l\in[n_l]}\;
	\bigcup_{\stackrel{\mathcal{I}_1\subseteq [n_1]}{\mathcal{I}_1\neq\emptyset}}
	\ldots
	\bigcup_{\stackrel{\mathcal{I}_{l-1}\subseteq [n_{l-1}]}{\mathcal{I}_{l-1}\neq\emptyset}}
	\Setbar{x\in\RR^d}{\sum_{i_0=1}^{d} \sum_{i_1\in\mathcal{I}_1}\ldots\sum_{i_{l-1}\in\mathcal{I}_{l-1}} \prod_{p=1}^{l} x_{i_0} (W_p)_{i_{p-1},i_p} = 0} .
    \end{align*}
    With probability 1 over $(W_l)_{l=1}^{k-1}$, the set of zeros of each polynomial inside the bracket above has measure zero.
    Since there are only finitely many such polynomials, one obtains $\gamma_d(B)=0$ .
    
    We are now ready to prove the lemma.
    From $\interior(P_i)\neq\emptyset$ and $\gamma_d(B\cup\partial S)=0$,
    it follows that
    \begin{align*}
	\interior(P_i)\cap (G\cup\interior(S))
	=\interior(P_i)\cap (\RR^d\setminus (B\cup\partial S))
	\neq\emptyset .
    \end{align*}
    For every $i\in[T],$ let $z_i\in \interior(P_i)\cap (G\cup\interior(S)).$
    Since $z_i\in\interior(P_i)$, it follows from \eqref{eq:Lip_op} that
    \begin{align*}
	\norm{g_k}_{\Lip} 
	=\max_{i\in[T]} \norm{A_i}_{\op}
	=\max_{i\in[T]} \norm{J(g_k)(z_i)}_{\op} .
    \end{align*}
    Now if $z_i\in\interior(S)$, then $J(g_k)(z_i)=0$, as $g_k$ is constant zero in a neighborhood of $z_i.$
    Otherwise, we must have $z_i\in G$, which implies $\mathcal{A}_{1\to k-1}(z_i)\in\Set{-1,+1}^{\sum_{l=1}^{k-1}n_l}.$
    Combining all these facts, we get
    \begin{align*}
	\norm{g_k}_{\Lip}
	=\;\max\limits_{z:\; \mathcal{A}_{1\to k-1}(z)\in\Set{-1,+1}^{\sum_{l=1}^{k-1}n_l}}\; \norm{J(g_k)(z)}_{\op} .
    \end{align*}
    Finally, the inequality $\norm{f_k}_{\Lip}\leq\norm{g_k}_{\Lip}$ follows from the 1-Lipschitz property of ReLU.

\subsection{Proof of Lemma \ref{lem:NN_PWL}}\label{app:lem:NN_PWL}
    Let $T=2^{\sum_{l=1}^{k} n_l},$
    and $\Set{\mathcal{A}_1,\ldots,\mathcal{A}_T}\in\Set{-1,+1}^{\sum_{l=1}^{k}n_l}$ 
    denote the set of all possible binary strings of dimension $\sum_{l=1}^{k}n_l$, where each entry takes value $-1$ or $+1.$ 
    Let us index the entries of each string by $\mathcal{A}_j=\Set{\mathcal{A}_{j,l,i_l}}_{l\in[k],i_l\in[n_l]}.$
    Let $P_j\subseteq\RR^d$ be the set of inputs 
    where the activation pattern of all neurons up to layer $k$
    matches perfectly with $\mathcal{A}_j$, namely
    \begin{align*}
	P_j
	&=\bigcap_{l\in[k]}\;\bigcap_{i_l\in[n_l]} \Setbar{x\in\RR^d}{g_{l,i_l}(x)\mathcal{A}_{j,l,i_l}\geq 0}\\
	&=\bigcap_{l\in[k]}\;\bigcap_{i_l\in[n_l]} \Setbar{x\in\RR^d}{ \sum_{i_0=1}^{d} \sum_{i_1=1}^{n_1}\ldots\sum_{i_{l-1}=1}^{n_{l-1}} \prod_{p=1}^{l} x_{i_0} (W_p)_{i_{p-1},i_p} \prod_{p=1}^{l-1} \mathbbm{1}_{\mathcal{A}_{j,p,i_p}> 0}\; \mathcal{A}_{j,l,i_l}\; \geq 0} .
    \end{align*}
    It is clear that $P_j$ is a polyhedron.
    Also, every coordinate function $f_{k,i_k}$ admits the following linear representation on $P_j$
    \begin{align*}
	f_{k,i_k}(x)=\sum_{i_0=1}^{d} \sum_{i_1=1}^{n_1}\ldots\sum_{i_{l-1}=1}^{n_{k-1}} \prod_{p=1}^{k} x_{i_0} (W_p)_{i_{p-1},i_p} \mathbbm{1}_{\mathcal{A}_{j,p,i_p}> 0},\quad\forall\,x\in P_j .
    \end{align*}
    This implies that $f_k$ coincides with a linear function on $P_j.$
    As every input must take one of the $T$ strings as an activation pattern, we also have $\RR^d=\cup_{i=1}^{T} P_j.$
    Thus according to Definition \ref{def:PWL}, $f_k$ is a PWL function.
    Similarly, $g_k$ is also piecewise linear.

\end{document}